\documentclass[journal]{IEEEtran}

\usepackage{amsmath,amsfonts,amssymb,amsthm,epsfig,epstopdf,url,array,xcolor,soul,subfigure}
 
\usepackage{mathtools}

\usepackage[noend]{algorithmic}
\usepackage[ruled,noend]{algorithm2e}
\usepackage{cite}

\newtheorem{theorem}{Theorem}

\newtheorem{problem}{Problem}
\newtheorem{remark}{Remark}
\newtheorem{definition}{Definition}

\PassOptionsToPackage{end}{algorithmic}

\newcommand{\alg}{\texttt{DRM}\xspace}
\newcommand{\dcp}{\texttt{DCP}\xspace}

\newcommand{\crel}{\texttt{central-robust}\xspace}
\newcommand{\cgr}{\texttt{central-greedy}\xspace}

\title{Distributed Attack-Robust Submodular Maximization for Multi-Robot Planning}

\author{Lifeng~Zhou,~\IEEEmembership{Member,~IEEE,} Vasileios~Tzoumas,~\IEEEmembership{Member,~IEEE,} George~J.~Pappas,~\IEEEmembership{Fellow,~IEEE,}  \\and Pratap~Tokekar,~\IEEEmembership{Member,~IEEE} 
\thanks{L. Zhou was with the Department of Electrical
and Computer Engineering, Virginia Tech, Blacksburg, VA, USA when part of the work was completed. He is currently with the GRASP Laboratory, University of Pennsylvania, Philadelphia, PA, USA (email: {\tt\small lfzhou@seas.upenn.edu}).}
\thanks{V. Tzoumas is with the Department of Aerospace Engineering, University of Michigan, Ann Arbor, MI 48109 USA (email: {\tt\small vtzoumas@umich.edu}).}
\thanks{G. J. Pappas is with the Department of Electrical and Systems Engineering, University of Pennsylvania, Philadelphia, PA 19104 USA (email: {\tt\small pappagsg@seas.upenn.edu}).}
\thanks{P. Tokekar is with the Department of
Computer Science, University of Maryland, College Park, MD 20742, USA (email: {\tt\small tokekar@umd.edu}).}
\thanks{This work is supported by the ARL CRA DCIST, the National Science Foundation under Grant No. 479615, and the Office of Naval Research under Grant No. N000141812829.}
}

\begin{document}
\maketitle

\begin{abstract}
In this paper, we design algorithms to protect swarm-robotics applications against sensor denial-of-service (DoS) attacks on robots.  
We focus on applications requiring the robots to jointly select actions, e.g., which trajectory to follow, among a set of available actions. Such applications are central in large-scale robotic applications, such as multi-robot motion planning for target tracking. But the current attack-robust algorithms are centralized.  In this paper, we propose a general-purpose distributed algorithm towards robust optimization at scale, with local communications only.  We name it \textit{Distributed Robust Maximization} (\alg).  \alg proposes a divide-and-conquer approach that distributively partitions the problem among cliques of robots. Then, the cliques optimize in parallel, independently of each other. We prove \alg achieves a close-to-optimal performance. We demonstrate \alg's performance in Gazebo and MATLAB simulations, in scenarios of \textit{active target tracking with swarms of robots}. In the simulations, \alg achieves computational speed-ups, being 1-2 orders faster than the centralized algorithms. \textit{Yet}, it nearly matches the tracking performance of the centralized counterparts. 
Since, \alg overestimates the number of attacks in each clique, in this paper we also introduce an \textit{Improved Distributed Robust Maximization} (\texttt{IDRM}) algorithm. \texttt{IDRM} infers the number of attacks in each clique less conservatively than \alg by leveraging 3-hop neighboring communications. We verify \texttt{IDRM} improves \texttt{DRM}'s performance in simulations. 
\end{abstract}


\section{Introduction}
Safe-critical tasks of surveillance and exploration  often require mobile agility, and fast capability to detect, localize, and monitor.
For example, consider the tasks:

\smallskip

\paragraph*{{\small\textbullet}~Adversarial target tracking} Track adversarial targets that move across an urban environment, aiming to escape~\cite{nieto2013multi};

\paragraph*{{\small\textbullet}~Search and rescue} Explore a burning building to localize any people trapped inside~\cite{kumar2017opportunities}.

\smallskip

Such scenarios can greatly benefit from teams of mobile robots that are agile, act as sensors, and plan their actions rapidly. For this reason, researchers are (i) pushing the frontier on robotic miniaturization and perception~\cite{nieto2013multi,kumar2017opportunities,michini2014robotic,karaman2012high,cadena2016past,8206030,8255576}, to enable mobile agility and autonomous sensing, and (ii) developing {distributed} coordination algorithms~\cite{atanasov2015decentralized,schlotfeldt2018anytime,khodayi2019distributed,corah2019distributed,best2019dec}, to enable {multi-robot planning}, i.e., the joint optimization of robots' actions. 

Particularly, distributed planning algorithms are important when one wishes to deploy large-scale teams of robots; e.g., at the swarm level of tens or hundreds of robots. One reason is that distributed algorithms scale better for larger numbers of robots than their centralized counterparts~\cite{atanasov2015decentralized}. {Additionally, in large-scale teams, it is infeasible for all robots to communicate with each other; typically, each robot can communicate only with the robots within a certain communication range.} 

\begin{figure}
\centering
\includegraphics[width=0.6\columnwidth]{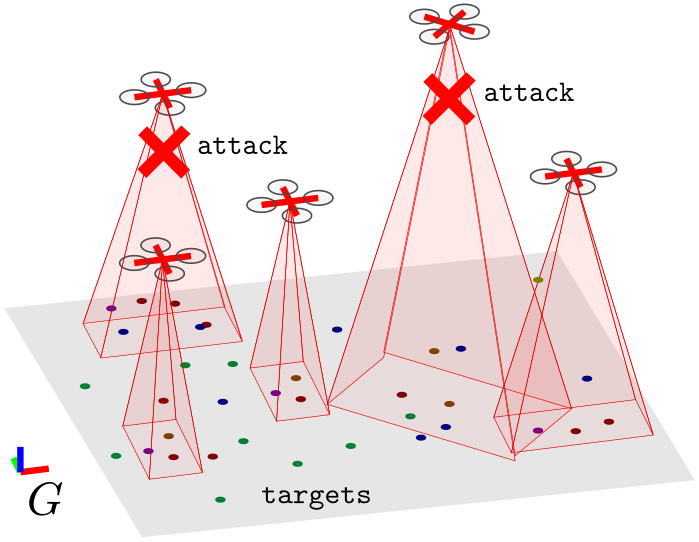}
\caption{Targets’ attacks can block robots’ field-of-view: in target tracking with aerial robots, the robots are mounted with down-facing cameras to track mobile targets (depicted as dots). 
\label{fig:uav_tracking}}
\end{figure}

However, the safety of the above critical scenarios can be at peril.  Robots operating in adversarial scenarios may get cyber-attacked or simply face failures, resulting in a temporary withdrawal of robots from the task (e.g., because of temporary deactivation of their sensors, blockage of their field of view, among others; {cf.~Fig.~\ref{fig:uav_tracking}}).  We refer to such attacks as \textit{Denial-of-Service} (DoS). Hence, in such adversarial scenarios, distributed {attack-robust} planning algorithms become necessary.\footnote{We henceforth consider the terms \textit{attack} and \textit{failure} equivalent, both resulting in robot withdrawals from the task at hand.} 

In this paper, we formalize a general framework for \textit{distributed attack-robust multi-robot planning} for tasks that require the maximization of submodular functions, such as active target tracking~\cite{dames2017detecting} and informative path planning~\cite{singh2009efficient} with multiple-robots.\footnote{Submodularity is a diminishing returns property~\cite{nemhauser1978analysis}, capturing the intuition that the more robots participate in a task, the less the gain (return) one gets by adding an extra robot towards the task.} We focus on \textit{worst-case} attacks that can result in up to $\alpha$ robot withdrawals from the task at each step.

Attack-robust multi-robot planning is computationally  hard and requires accounting  for  all  possible $\alpha$ withdrawals, a problem of combinatorial complexity.  Importantly, even in  the presence of no withdrawals, the problem of multi-robot planning is NP-hard~\cite{feige1998threshold,tokekar2014multi}.  All in all, the necessity for distributed attack-robust algorithms, and the inherent computational hardness motivates our goal in this paper: to provide a distributed, provably  close-to-optimal  approximation algorithm.  To this  end,  we  capitalize  on  recent  algorithmic  results  on {centralized} attack-robust multi-robot~planning~\cite{tzoumas2017resilient,tzoumas2018resilient,schlotfeldt2018resilient,zhou2018resilient} and present a distributed attack-robust algorithm.

\textbf{Related work.} Researchers have developed several distributed, but attack-free, planning algorithms, such as~\cite{atanasov2015decentralized,schlotfeldt2018anytime,khodayi2019distributed,corah2019distributed,best2019dec}.  For example,~\cite{atanasov2015decentralized} developed a decentralized algorithm, building on the local greedy algorithm proposed in~\cite[Section~4]{fisher1978analysis}, which guarantees a $1/2$ suboptimality bound for submodular objective functions.  In~\cite{atanasov2015decentralized} the robots form a {string} communication network, and sequentially choose an action, given the actions of the robots that have chosen so far. In~\cite{corah2019distributed}, the authors proposed a speed-up of~\cite{atanasov2015decentralized}'s approach, by enabling the greedy sequential optimization to be executed over directed acyclic graphs, instead of {string ones}.   In scenarios where the robots cannot observe all the chosen actions so far, distributed, but still attack-free, algorithms for submodular maximization are developed  in~\cite{gharesifard2018distributed,grimsman2018impact}.  Other distributed, attack-free algorithms are developed in the machine learning literature on submodular maximization, towards sparse selection (e.g., for data selection, or sensor placement)~\cite{mirzasoleiman2013distributed}, instead of planning.

Researchers have also developed robust planning algorithms~\cite{saulnier2017resilient,saldana2017resilient, mitra2019resilient,schlotfeldt2018resilient,zhou2018resilient,tzoumas2017resilient,tzoumas2018resilient}.  Among these, \cite{saulnier2017resilient,saldana2017resilient,mitra2019resilient} focus on deceptive attacks, instead of DoS attacks.  In more detail, \cite{saulnier2017resilient,saldana2017resilient} provide distributed resilient formation control algorithms to deal with malicious robots in the team, that share deceptive information. Similarly, \cite{mitra2019resilient} provides a distributed resilient state estimation algorithm against robots executing Byzantine attacks by sending differing state
estimates or not transmitting any estimates. In contrast to~\cite{saulnier2017resilient,saldana2017resilient,mitra2019resilient}, the papers \cite{schlotfeldt2018resilient,zhou2018resilient} consider DoS attacks in multi-robot motion planning but in centralized settings: \cite{schlotfeldt2018resilient} provides centralized attack-robust algorithms for active information gathering, and \cite{zhou2018resilient} for target tracking.  The paper \cite{9450823} extended \cite{schlotfeldt2018resilient} to a Model Predictive Control (MPC) framework. The algorithms are based on the centralized algorithms proposed in~\cite{tzoumas2017resilient,tzoumas2018resilient}.  Additional attack-robust algorithms are proposed in~\cite{orlin2018robust,bogunovic2017distributed}, which, however, are valid for only a limited number of attacks. {For a thorough literature review on attack-robust combinatorial algorithms, we refer the reader to~\cite{tzoumas2019robust}.}

\textbf{Contributions.}  {Towards enabling attack-robust planning in multi-robot scenarios requiring local communication among robots, we make the following contributions:}
\smallskip

A.~~We present the algorithm \textit{Distributed Robust Maximization} (\alg):  \alg distributively partitions the problem among cliques of robots, where all robots are within communication range. Then, naturally, the cliques optimize in parallel, using~\cite[Algorithm~1]{zhou2018resilient}.
    We prove (Section~\ref{sec:perform}):
    
    \begin{itemize}
    \item \textit{System-wide attack-robustness}: \alg is valid for any number $\alpha$ of worst-case attacks; 
    
    \item \textit{Computational speed-up}: \alg is faster up to a factor $1/K^2$ than its centralized counterparts in~\cite{zhou2018resilient}, {where $K$ is the number of cliques chosen by \alg. {$K$ depends on the inter-robot communication range, denoted henceforth by $r_c$, which is given as input to \alg (and, as a result, by changing $r_c$ one  controls $K$)}.}
    
    \item \textit{Near-to-centralized \!approximation\! performance}: 
    Even though \alg is a distributed algorithm, and faster than algorithm its centralized counterpart~\cite[Algorithm~1]{zhou2018resilient}, \alg achieves a near-to-centralized performance, having a suboptimality bound equal to~\cite[Algorithm~1]{zhou2018resilient}'s. 
    \end{itemize}
    
B.~~We design an \textit{Improved Distributed Robust Maximization} (\texttt{IDRM}) algorithm to relax the conservativeness of \texttt{DRM} in inferring the number of attacks against each cliques  (Section~\ref{sec:improve_main}). {Specifically, \texttt{DRM} assumes that the number of attacks against each clique is equal to $\alpha$, the \textit{total} number of attacks against all robots.} {Instead, in \texttt{IDRM}}, robots communicate with their 3-hop neighbors {in the communication network of all robots} to infer the number of attacks.{\footnote{A robot's 3-hop neighbors are its neighbors, its neighbors' neighbors, and its neighbors' neighbors' neighbors.}} 
    \texttt{IDRM} has the same approximation performance as  \texttt{DRM}. {In our numerical evaluations, it maintains a comparable running time to \alg (Section~\ref{sec:performance-IDRM})}. 
    
\smallskip
Notably, the proposed algorithms in this paper allow for the communication graph to be disconnected. 

\textbf{Numerical evaluations.} First, we present Gazebo and MATLAB evaluations of \alg, in scenarios of \textit{active target tracking with swarms of robots} (Section~\ref{subsec:sim_target_tracking}). 
All simulation results demonstrate \alg's computational benefits: \alg runs 1 to 2 orders faster than its centralized counterpart in~\cite{zhou2018resilient}, achieving running times 0.5msec to 15msec for up to 100 robots.  And, yet, \alg exhibits negligible deterioration in performance (in terms of number of targets covered). Second, we compare the performance of \texttt{DRM} and \texttt{IDRM} through Matlab simulations (Section~\ref{subsec:improve}). We show that \texttt{IDRM} performs better than \texttt{DRM} in practice (higher target coverage), {while maintaining a comparable} running time.

\textbf{Comparison with the preliminary results.} Preliminary results were presented in~\cite{zhou2019approximation,zhou2019distributed}. In this paper, we present for the first time the algorithm \textit{Improved Distributed Robust Maximization} (Section~\ref{sec:improve_main}). Moreover, the corresponding Matlab simulations are new (Section~\ref{subsec:improve}). Also, we present a formal analysis of \alg's computation and approximation  performance (Appendix-A \& B). Further, we analyze the approximation of a myopic algorithm and compare its practical performance with \texttt{DRM} and \texttt{IDRM} through Matlab simulations (Remark~\ref{rem:myopic} and Appendix C).

\section{Problem Formulation}\label{sec:problem}

\begin{figure}[t]
\centering{
{\includegraphics[width=0.6\columnwidth,trim= 0cm .0cm 0 0cm,clip]{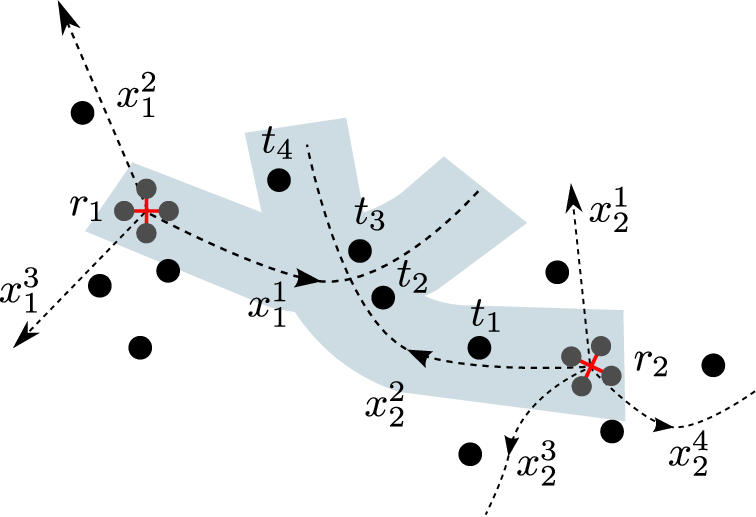}}
\caption{Robots choose trajectories from a set of motion primitives: at each time step, each robot has a set of motion primitives to choose as its trajectory (each covering different targets, depicted as dots). For example, robot 1 has 3 motion primitives, $\{x_1^1, x_1^2, x_1^3\}$, and robot 2 has 4 motion primitives, $\{x_2^1, x_2^2, x_2^3, x_2^4\}$, where $x_1^1$ covers 2 targets, $\{t_2, t_3\}$, and $x_2^2$ covers 4 targets, $\{t_1, t_2, t_3, t_4\}$. In combination, however, the two motion primitives totally cover 4 targets, $\{t_1, t_2, t_3, t_4\}$. \label{fig:uav_trajectory}}
}
\end{figure}

We formalize the problem of \textit{distributed attack-robust submodular maximization for multi-robot planning}.  At each time-step, the problem {asks for assigning} actions to the robots, to maximize an objective function despite attacks.  For example, in the active target tracking with aerial robots (see Fig.~\ref{fig:uav_tracking}), the robots' possible actions are their motion primitives; the objective function is the number of covered targets; and the attacks are field-of-view blocking attacks.

We next introduce our framework in more detail:\footnote{\textbf{Notations.} Calligraphic fonts denote sets (e.g., $\mathcal{A}$).  $2^{\mathcal{A}}$ denotes $\mathcal{A}$'s power set, and $|\mathcal{A}|$ its cardinality.  $\mathcal{A}\setminus\mathcal{B}$ are the elements in $\mathcal{A}$ not in~$\mathcal{B}$.}

\paragraph{Robots} We consider a multi-robot team $\mathcal{R}$.  At a given time-step, $p_i$ is robot $i$'s position in the environment ($i \in \mathcal{R}$).  We define $\mathcal{P}\triangleq\{p_1,\ldots,p_{|\mathcal{R}|}\}$ with $|\mathcal{R}|=N$. 

\paragraph{Communication graph} Each robot communicates only with those robots within a prescribed \textit{communication range}. Without loss of generality, we assume all robots to have the same communication range $r_c$.  That way, an (undirected) \textit{communication graph} $G=\{\mathcal{R},\mathcal{E}\}$ is induced, with nodes the robots $\mathcal{R}$, and edges $\mathcal{E}$ such that $(i,j)\in\mathcal{E}$ if and only if $\|p_i-p_j\|_2\leq r_c$.  The \textit{neighbors} of robot $i$ are all robots within the range $r_c$, and are denoted by~$\mathcal{N}_i$.

\paragraph{Action set} Each robot $i$ has an available set of \emph{actions} to choose from; we denote it by $\mathcal{X}_i$.  The robot can choose at most $1$ action at each time, due to operational constraints; e.g., in motion planning, $\mathcal{X}_i$ denotes robot $i$'s motion primitives, and the robot can choose only $1$ motion primitive at a time to be its trajectory. For example, in Fig.~\ref{fig:uav_trajectory} we have 2 robots, where $\mathcal{X}_1 = \{x_1^1, x_1^2, x_1^3\}$ (and robot 1 chooses $x_1^1$ as its trajectory) and $\mathcal{X}_2 = \{x_2^1, x_2^2, x_2^3, x_2^4\}$ (and robot 2 chooses $x_2^2$ as its trajectory). We let $\mathcal{X} \triangleq \bigcup_{i\in \mathcal{R}} \mathcal{X}_i$.  Also, $\mathcal{S}\subseteq \mathcal{X}$ denotes a valid assignment of actions to all robots. For instance, in  Fig.~\ref{fig:uav_trajectory}, $\mathcal{S}=\{x_1^1, x_2^2\}$. 

\paragraph{Objective function} The quality of each $\mathcal{S}$ is quantified by a non-decreasing and submodular function $f: 2^{\mathcal{X}} \to \mathbb{R}$. For example, this is the case in {active target tracking with mobile robots}, when $f$ is the number of covered targets~\cite{tokekar2014multi}. As shown in Fig.~\ref{fig:uav_trajectory}, the number of targets covered by the chosen actions, $\mathcal{S}=\{x_1^1, x_2^2\}$, is $f(\mathcal{S}) = 4$. 

\paragraph{Attacks} At each time step, we assume the robots encounter worst-case sensor DoS attacks. Each attack results in a robot's sensing removal (at least temporarily, i.e., for the current time step). In this paper, we study the case where the maximum number of attacks at each time step is known, per Problem~\ref{pro:dis_resi_sub} below. We denote the maximum number of attacks by $\alpha$. We also assume that if a robot encounters a sensor DoS attack, it can still act as an information relay node to communicate with other robots. 

\begin{problem}[Distributed attack-robust submodular maximization for multi-robot planning]\label{pro:dis_resi_sub}
{At each time step}, the robots, by exchanging information only over the communication graph $G$, assign an action to each robot $i \in \mathcal{R}$ to maximize $f$ against $\alpha$ worst-case attacks/failures:
\begin{align} \label{eq:dis_resi_sub}
\begin{split}
\max_{\mathcal{S}\subseteq \mathcal{X}}\;&\;\min_{\mathcal{A}\subseteq \mathcal{\mathcal{S}}}\;\; f(\mathcal{S}\setminus\mathcal{A})\\
\emph{\text{s.t. }}\;\;&|\mathcal{S}\cap \mathcal{X}_i|= 1, ~\text{for all } i\in \mathcal{R}, |\mathcal{A}| \leq \alpha,\\
\end{split}
\end{align}
where $\mathcal{A}$ corresponds to the actions of the attacked robots, and $\alpha$ is assumed known.\footnote{{The constraint in eq.~\eqref{eq:dis_resi_sub} ensures that each robot chooses $1$ action per time step (e.g., $1$ trajectory among a set of primitive trajectories.)}}
\end{problem}

Problem~\ref{pro:dis_resi_sub} is equivalent to a two-stage perfect information sequential game~\cite[Chapter 4]{myerson2013game} between the robots and an attacker. Particularly, the robots  first  select  $\mathcal{S}$, and, then, the attacker, \textit{after observing $\mathcal{S}$}, selects the worst-case $\mathcal{A}$ (which is unknown to the robots). The ``$\min$'' operator means that the attacker aims to minimize the robots' objective function $f$ by removing up to $\alpha$ robots' actions $\mathcal{A}$, no matter what actions $\mathcal S$ the robots take. This is the worst-case attack on the robots' actions $\mathcal S$. The ``$\min$'' operation is from the attackers' point of view, and the robots do not know which robots (or their actions $\mathcal A$) will be attacked. Under this assumption, the robots aim to maximize the objective function $f$ by assuming the attacker executes up to $\alpha$ worst-case attacks. 

Notably, a robot does not have to understand if it encounters a DoS attack or not, and it does not need to communicate such information to its neighbors either.  That is because the robots' planning happens before the attacks take place. So the robots do not even have a way of knowing who gets a DoS attack, and that is why they need to consider the worst-case scenario. The DoS attacks are not permanent, and a rational attacker will choose the worst-case subset of sensors to attack at each time step.  In other words, Problem~\ref{pro:dis_resi_sub} asks for planning algorithms to prepare the robot team to be robust to withstand the worst-case attacks \textit{without} knowing the attack details (e.g., which robots' sensors get attacked).

\section{A Distributed Algorithm: \alg}~\label{sec:algorithm}
We present \textit{Distributed Robust Maximization} (\alg), a distributed algorithm for Problem~\ref{pro:dis_resi_sub} {for the case where $\alpha$ is known} (Algorithm~\ref{alg:DRA}). \alg  executes sequentially two main steps: \textit{distributed  clique partition} (\alg's line~1), and \textit{per clique attack-robust optimization} (\alg's lines~2-8). During the first step, the robots communicate with their neighbors to partition $G$ into cliques of \textit{maximal} size (using Algorithm~\ref{alg:EDNCCA}, named \texttt{DCP} in \alg's line~1).\footnote{A clique is a set of robots that can all communicate with each other.}   During the second step, each clique computes an attack-robust action assignment (in parallel with the rest), using the centralized algorithm in~\cite{zhou2018resilient} ---henceforth, we refer to the algorithm in~\cite{zhou2018resilient} as \crel.  \crel takes similar inputs to \alg: a set of actions, a function, and a number of attacks.

We describe \alg's two steps in more detail below; and quantify its running time and performance in Section~\ref{sec:perform}.

\begin{algorithm}[t]
\caption{Distributed robust maximization (\alg).}
\begin{algorithmic}[1]
\REQUIRE
    Robots' available actions $\mathcal{X}_i$,  $i\in \mathcal{R}$;
    monotone and submodular function $f$;
    attack number $\alpha$.
\ENSURE  Robots' actions $\mathcal{S}$.
\STATE Partition $G$ to cliques $\mathcal{C}_1,\ldots, \mathcal{C}_K$ by calling $\texttt{DCP}(\mathcal{P}, r_c)$; 
\STATE $\mathcal{S}_k \leftarrow \emptyset$ for all $k = \{1, \ldots, K\}$; ~\label{}
\FOR{\textbf{each clique $\mathcal{C}_k$ in parallel,}}
\label{DRA: for_cliq_start}
\IF{$\alpha < |\mathcal{C}_k|$}
\label{DRA: alphaless}
\STATE $\mathcal{S}_k=$ \crel\!\!$(\bigcup_{i\in \mathcal{C}_k}\mathcal{X}_i,f, \alpha)$; \label{DRA:alpha_small}
\ELSE \label{DRA:alpha_less}
\STATE $\mathcal{S}_k=$ \crel\!\!$(\bigcup_{i\in \mathcal{C}_k}\mathcal{X}_i,f, |\mathcal{C}_k|)$;
\label{DRA:alpha_large}
\ENDIF
\ENDFOR
\label{DRA: for_cliq_end}
\RETURN $\mathcal{S} = \bigcup_{k=1}^{K} \mathcal{S}_k$.
\label{DRA:union}
\end{algorithmic}
\label{alg:DRA}
\end{algorithm}

\subsection{Distributed clique partition (Step-1 of \alg)}~\label{subsec:clique_cover}
We present the first step of \alg, \textit{distributed clique partition}, which is inspired by \cite[Algorithm 2]{pattabiraman2013fast} (\alg's line~1, that calls \texttt{DCP}, whose pseudo-code is presented in Algorithm~\ref{alg:EDNCCA}).  The problem of distributed clique partition is inapproximable in polynomial time, since finding even a single clique of a desired size is inapproximable~\cite{zuckerman2006linear}, {even by centralized algorithms.}  

\texttt{DCP} requires three rounds of communications among neighboring robots to form separate cliques. In the first round, each robot $i$ communicates to find its neighbors (Algorithm~\ref{alg:EDNCCA}, line~\ref{line:encc_find_nei}). In the second round, it shares its augmented neighbor set $\mathcal{N}_i^+$ (containing its neighbors and itself) with its neighbors, and receives its neighbors' sets $\{\mathcal{N}_j^+\}$ (Algorithm~\ref{alg:EDNCCA}, line~\ref{line:encc_share_rec_nei}). Then robot $i$ intersects its augmented neighbor set $\mathcal{N}_i^+$ with 
each of its neighbors' augmented neighbor sets $\mathcal{N}_j^+$, and sets the largest intersection as its clique (Algorithm~\ref{alg:EDNCCA}, line~\ref{line:encc_largest}). 

The aforementioned clique computation of \texttt{DCP} differs from  \cite[Algorithm 2]{pattabiraman2013fast} in that in \cite[Algorithm 2]{pattabiraman2013fast} {each robot $i$ computes its clique as the intersection of $\mathcal{N}_i^+$ and $\mathcal{N}_j^+$ where $j$ is the neighbor with the largest degree in $\mathcal{N}_i$, {whereas in \texttt{DCP}'s line~5,} 
each robot $i$
computes its clique as the intersection of $\mathcal{N}_i^+$and $\mathcal{N}_j^+$, where $j$ instead is the neighbor with the largest neighborhood overlap with $i$.  That way, \texttt{DCP} is more likely to obtain larger cliques for each robot.   Also, the cliques returned by \cite[Algorithm 2]{pattabiraman2013fast} can overlap with each other. In order to form separate cliques, \texttt{DCP} executes the third round of communication to share the computed cliques among neighbors (Algorithm~\ref{alg:EDNCCA}, line~\ref{line:encc_share_rec_unique}). Specifically, each robot $i$ tells its neighbors which clique it will join. If the clique of some neighboring robot $j$ contains robot $i$ but robot $i$ chooses to join a different clique (by Algorithm~\ref{alg:EDNCCA}, line~\ref{line:encc_largest}), its neighboring robot $j$ will update its clique by removing robot $i$ from it. In this way, each robot $i$ will eventually belong to a single clique, and thus non-overlapping cliques can be generated. An illustrative example of \texttt{DCP} is shown in Fig.~\ref{fig:clique_partition}.

\begin{remark}\label{rem:clique}
We partition the communication graph $G$ into cliques instead of connected subgroups because we want robots to have a shorter communication time interval before making decisions. Robots within a clique have point-to-point communication, while robots in a connected subgroup need multi-hop communication to propagate information. Typically, multi-hop communication takes time when the diameter of the communication graph is large. More formally, point-to-point communication has $O(1)$ communication round, while multi-hop communication has $O(N)$ communication rounds in the worst case (e.g., a line graph). Therefore, we choose the clique partition model that enables robots to have a shorter communication time before making decisions.

Reducing the time of information sharing becomes essential when robots need to make decisions quickly. For example, in a multi-target tracking scenario, robots need to firstly share information (e.g., subsets of targets covered) and then decide their actions (e.g., movements). Since targets are mobile and may move fast, the robots need to choose actions as quickly as they can. Otherwise, the information shared can be outdated and thus misleading for decision-making. Therefore, we use cliques with one-hop communication only to shorten the time of information sharing among robots. 


Moreover, multi-hop communication may reduce communication bandwidth and increase energy consumption. For example, when robots communicate over a single channel, only one robot can be active to broadcast information during a certain time interval. Therefore, communicating messages over multiple hops result in a smaller communication bandwidth and a higher communication cost~\cite{zanjireh2015survey}. But investigating the number of communication hops to trade off energy consumption/bandwidth with performance could be an interesting future direction.

\end{remark}

\begin{figure}[t]
\centering{
{\includegraphics[width=0.65\columnwidth,trim= 0cm .0cm 0 0cm,clip]{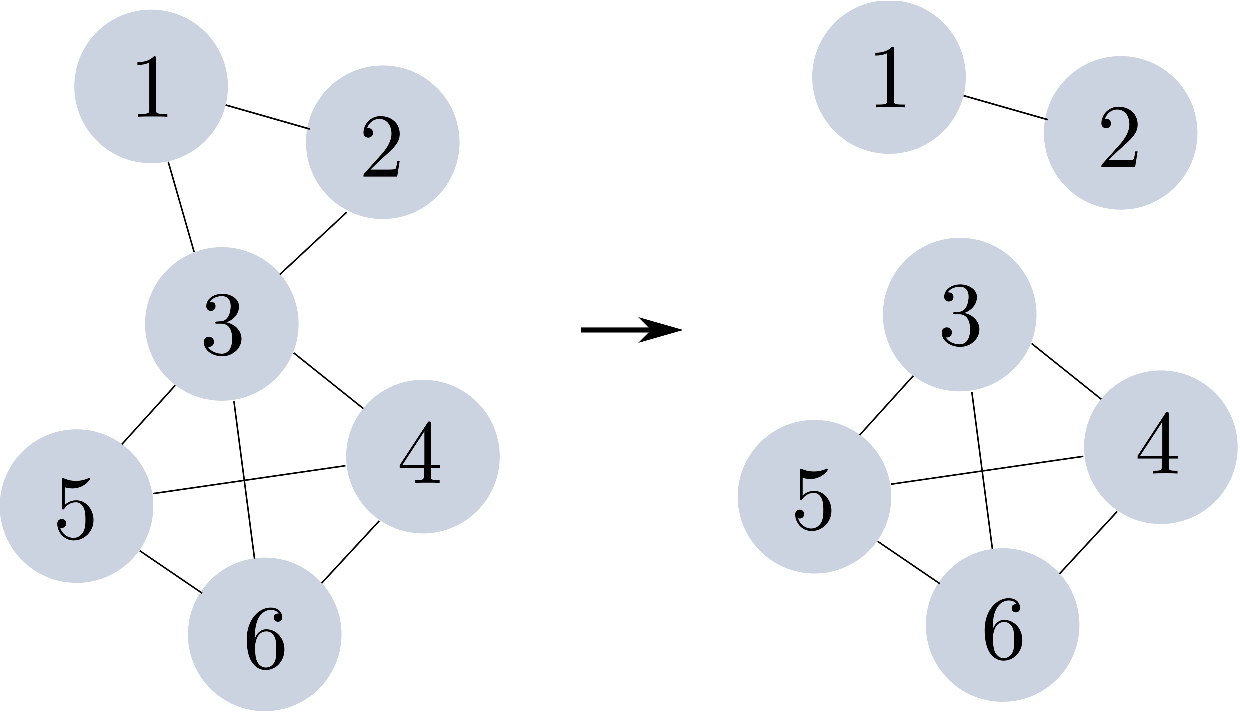}}
\caption{\texttt{DCP} partitions a graph of 6 robots into 2 separate cliques. Particularly, after clique computation, robots $1\sim6$ obtain cliques \{1, 2, 3\}, \{1, 2, 3\}, \{3, 4, 5, 6\}, \{3, 4, 5, 6\}, \{3, 4, 5, 6\}, and \{3, 4, 5, 6\}. Then in the third communication round, robot 3 shares its cliques with its neighbors (i.e., tells its neighbors \{1, 2\} that it joins clique \{3, 4, 5, 6\}), and robots 1 and 2 reset their cliques as \{1, 2\} and \{1, 2\}. }
\label{fig:clique_partition}
}
\end{figure}

}

\begin{figure*}[t]
\centering{
\subfigure[A communication graph $G$ of 15 robots]{\includegraphics[width=0.6\columnwidth]{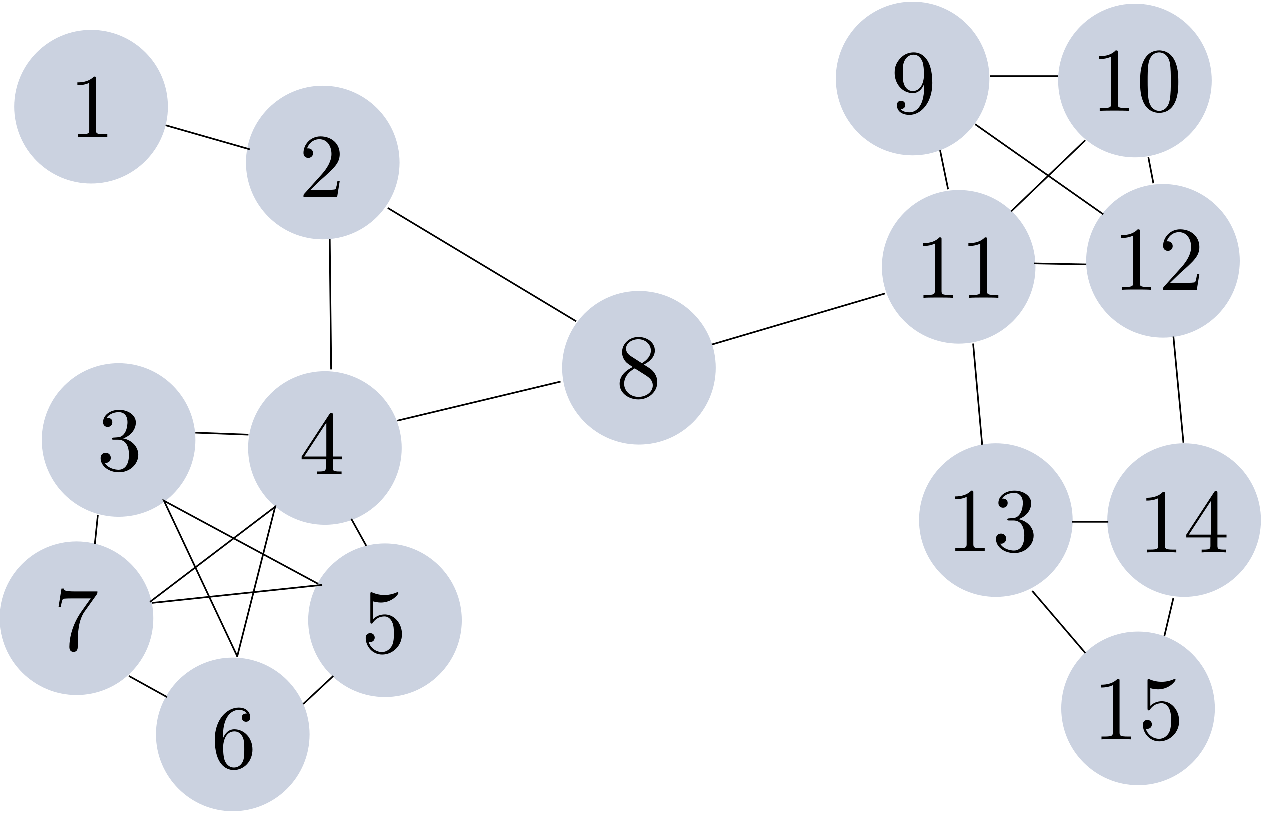}}~~~~~~~~~~~\subfigure[\dcp partitions $G$ into 5 cliques]{\includegraphics[width=0.6\columnwidth]{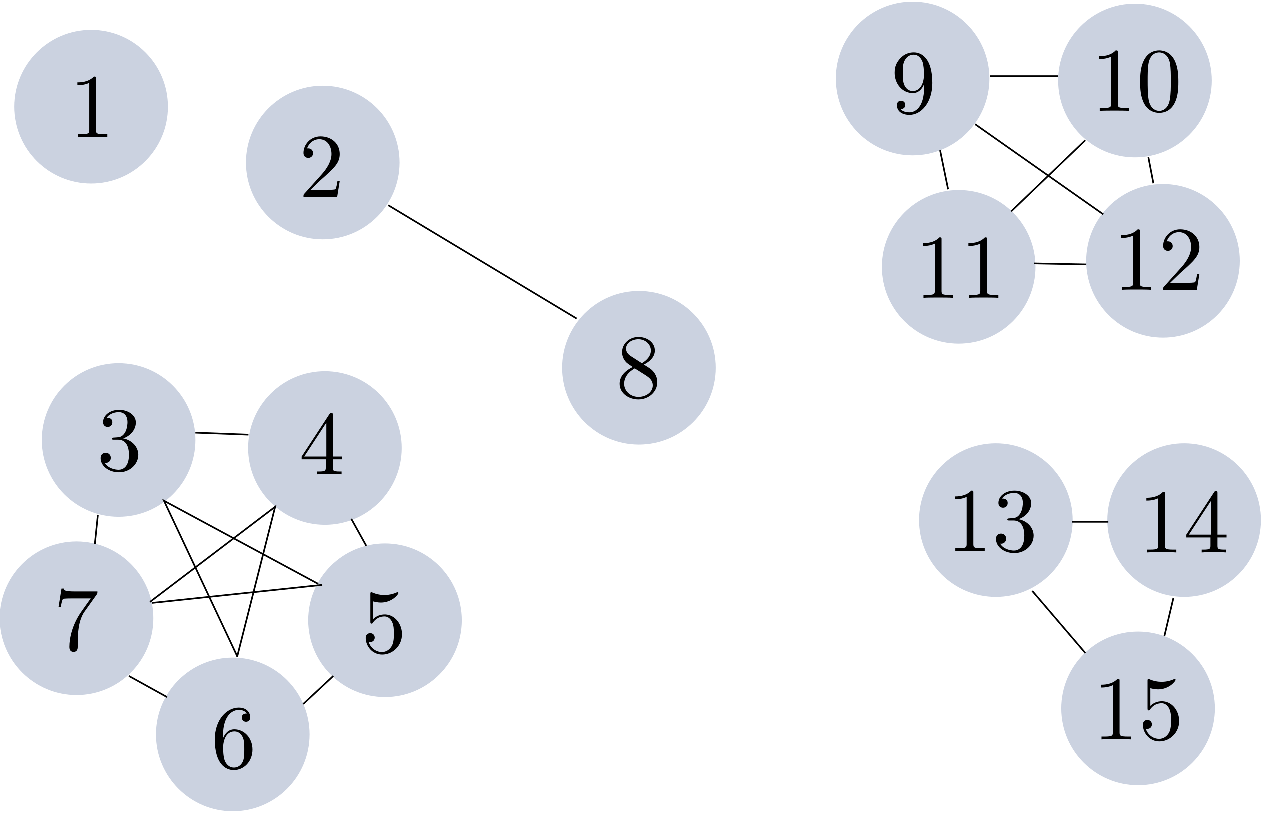}}~~~~~~~~~~~\subfigure[Each clique runs \crel]{\includegraphics[width=0.6\columnwidth]{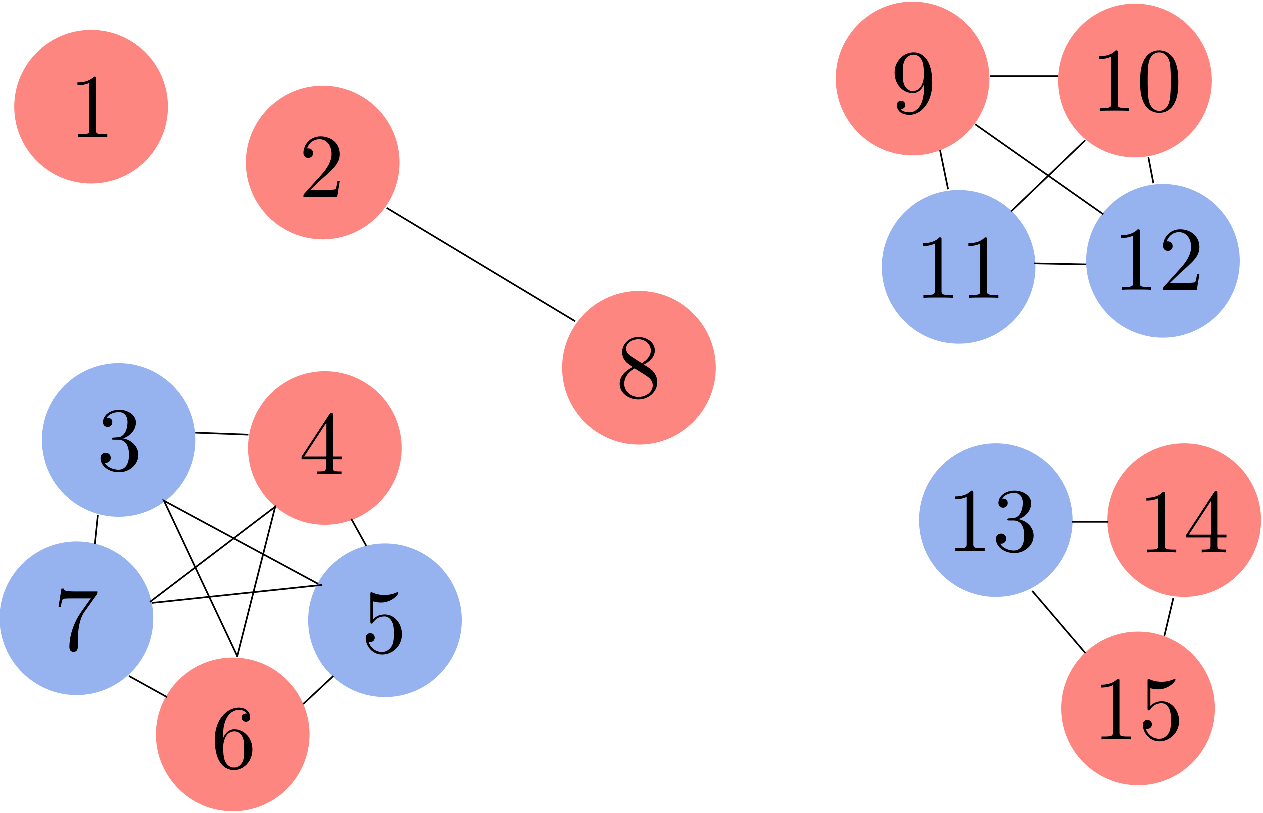}}
\caption{Qualitative description of \alg's steps over the communication graph $G$ in subfigure (a), composed of 15 robots. The number of conjectured attacks is considered to be $\alpha=2$. In the first step, we assume \dcp (\alg's line~1) partitions $G$ into 5 cliques, as shown in subfigure (b). In the second step, all 5 cliques perform \crel in parallel. Particularly, the cliques $\{(1), (2,8)\}$, since $\alpha$ is larger than or equal to their size,  consider that all of their robots will be attacked, and as a result they select all of their robots as baits (depicted with red in subfigure (c)), per \crel. In contrast, the remaining $3$ cliques, 
since  $\alpha$ is smaller than their size, they select $\alpha$ of their robots as baits. The remaining robots (depicted with blue in subfigure (c)) of each clique choose their actions greedily, independently of the other cliques, and assuming that the red robots in their clique do not exist.  
\label{fig:alg_step}}}
\end{figure*}

\subsection{Per clique attack-robust optimization (Step-2 of \alg)}~\label{subsec:DRA}
We now present \alg's second step: \textit{per clique  attack-robust  optimization} (\alg's lines~2-8).  
Since the step calls \crel as subroutine, we recall here its steps from~\cite{zhou2018resilient}: \crel takes as input the available actions of a set of robots $\mathcal{R}'\subseteq \mathcal{R}$ (i.e., the $\bigcup_{i\in\mathcal{R}'}\mathcal{X}_i$), a monotone submodular $f$, and a number of attacks $\alpha'\leq \alpha$, and constructs an action assignment $\mathcal{S}'$ by following a two-step process. First, it tries to approximate the conjectured worst-case attack to $\mathcal{S}'$, and, to this end, builds a  ``bait'' set as part of $\mathcal{S}'$.  Particularly, the bait set is aimed to attract all attacks at $\mathcal{S}'$, and for this reason, it has cardinality $\alpha'$ (the same as the number of conjectured attacks).  In more detail, \crel includes an action $x \in \bigcup_{i\in\mathcal{R}'}\mathcal{X}_i$ in the bait set (at most 1 action per robot, per Problem~\ref{pro:dis_resi_sub}) only if $f(\{x\})\geq f(\{x'\})$ for any other $x' \in \bigcup_{i\in\mathcal{R}'}\mathcal{X}_i$. That is, the bait set is composed of the ``best'' $\alpha'$ single actions. In the second step, \crel a) assumes the robots in the bait set are removed from $\mathcal{R}'$, and then b) greedily assigns actions to the rest of the robots using the centralized greedy algorithm in~\cite[Section~2]{fisher1978analysis} which ensures a near-optimal assignment (at least 1/2 close to the optimal). Specifically, for the remaining $|\mathcal{R}'|-\alpha'$ robots, denoted by $\mathcal{R}_a'$, the centralized greedy algorithm assigns actions $\mathcal{S}'_g$ for them as follows.
\begin{algorithmic}[1]
\STATE $\mathcal{S}'_g \leftarrow \emptyset$; 
\WHILE{$\mathcal{R}_a' \neq \emptyset$}
\STATE $(x_g', i_g') \in \underset{x \in \bigcup_{i\in\mathcal{R}_a'}\mathcal{X}_i}{\text{argmax}}~f(\mathcal{S}_g' \cup \{x\}) - f(\mathcal{S}_g')$; 
\STATE  $\mathcal{S}'_g \leftarrow \mathcal{S}'_g \cup\{x_g'\}$; $\mathcal{R}_a'\leftarrow \mathcal{R}_a'\setminus \{i_g'\}$. 
\ENDWHILE
\end{algorithmic}
Notably, it chooses an action $x_g'$ and the corresponding robot $i_g'$ with the maximal marginal contribution to the function value at each round (line 3).   

The bait set is one way of approximating $\alpha'$ worst-case attacks to the action assignment $\mathcal{S}'$. By executing the worst-case attacks, the attacker may or may not attack the bait set. Selecting the bait action set for the robots can add robustness against the worst-case attacks. That is, no matter the attacker attacks the bait set or not, \crel, composed of bait action assignment and greedy action assignment, gives suboptimality guarantees against the worst-cast attacks~\cite{tzoumas2017resilient, tzoumas2018resilient}.

In this context, \alg's second step is as follows: assuming the clique partition step returns $K$ cliques (\alg's line~1), now each clique in parallel with the others computes an attack-robust assignment for its robots using \crel (\alg's lines~3-8).  To this end, the cliques need to assess how many of the $\alpha$ attacks each will incur.  If there is no prior on the attack generation mechanism, then we consider a worst-case scenario where each clique incurs all the $\alpha$ attacks. Otherwise, we consider there is a prior on the attack mechanism such that each clique $k$ infers it will incur $\alpha_k\leq \alpha$ attacks. Without loss of generality, in \alg's pseudo-code in Algorithm~\ref{alg:DRA} we present the former scenario, where $\alpha_k=\alpha$ across all cliques; notwithstanding,  our theoretical results on \alg's performance  (Section~\ref{sec:perform}) hold for any $\alpha_k$ such that $\sum_{k=1}^K\alpha_k\geq \alpha$. Overall, \alg's lines~3-8 are~as~follows (cf.~example in Fig.~\ref{fig:alg_step}):

\begin{algorithm}[t]
\caption{
Distributed clique partition
(\texttt{DCP}).}
\begin{algorithmic}[1]
\REQUIRE 
Robots' positions $\mathcal{P}$; communication range $r_c$.
\ENSURE  Clique partition of graph $G$.
\STATE Given $\mathcal{P}$ and $r_c$, find communication graph $G$;
\FOR{each robot $i$} \label{line:encc_forstart}
\STATE Find robot $i$'s neighbor set $\mathcal{N}_i$ within $r_c$; {$\{$\color{gray}1st round communication\color{black}$\}$} \label{line:encc_find_nei}
\STATE Share $\mathcal{N}_i^{+} : = \{i, \mathcal{N}_i\}$ with robot $i$'s neighbors, and receives all $\mathcal{N}_j^{+}$ from its neighbors, $j\in\mathcal{N}_i$; {$\{$\color{gray}2nd round communication\color{black}$\}$} \label{line:encc_share_rec_nei}
\STATE Intersects $\mathcal{N}_i^{+}$ with every $\mathcal{N}_j^{+}$, and set the largest intersection as robot $i$'s clique, i.e., $\mathcal{C}^{i} = \text{argmax}_{\mathcal{N}_i^{+} \bigcap \mathcal{N}_{j}^{+}} |\mathcal{N}_i^{+} \bigcap \mathcal{N}_{j}^{+}|, ~j\in\mathcal{N}_i$; \label{line:encc_largest}
\STATE Share $\mathcal{C}^{i}$ with robot $i$'s neighbors;
 {$\{$\color{gray}3rd round communication\color{black}$\}$} 
\label{line:encc_share_rec_unique}
\ENDFOR \label{line:encc_endfor}
\RETURN {Generated cliques.}
\end{algorithmic}
\label{alg:EDNCCA}
\end{algorithm}

\paragraph{\alg's lines~\ref{DRA: alphaless}-\ref{DRA:alpha_small} ($\alpha< |\mathcal{C}_k|$)} If $\alpha$ is less than the clique's size (\alg's line~\ref{DRA: alphaless}), then the clique's robots choose actions by executing \crel on the clique assuming $\alpha$ attacks (\alg's line~\ref{DRA:alpha_small}).

\paragraph{\alg's lines~\ref{DRA:alpha_less}-\ref{DRA:alpha_large} ($\alpha\geq |\mathcal{C}_k|$)} But if $\alpha$ is larger than the clique's size  (\alg's line~\ref{DRA:alpha_less}), then the clique's robots choose actions by executing \crel on the clique assuming $|\mathcal{C}_k|$ attacks (\alg's line~\ref{DRA:alpha_small}); i.e., assuming that all clique's robots will be attacked. 

\paragraph{\alg's line~\ref{DRA:union}} All in all, now all robots have assigned actions, and $\mathcal{S}$ is the union of all assigned actions across all cliques (notably, the robots of each clique $k$ know only $\mathcal{S}_k$, where $\mathcal{S}_k$ is per the notation in \alg).

\medskip
Finally, \alg is valid for any number of attacks.

\section{Performance Analysis of \texttt{DRM}}~\label{sec:perform}

We now quantify \alg's performance, by bounding its
computational and approximation performance. To this end,
we use the following notion of curvature for set functions.

\subsection{Curvature}\label{subsec:curv}

\begin{definition}[Curvature~\cite{conforti1984submodular}]\label{def:curvature}
Consider non-decreasing submodular $f:2^{\mathcal{X}}\mapsto\mathbb{R}$ such that $f({x})\neq 0$, for any $x \in \mathcal{X}\setminus \{\emptyset\}$ (without loss of generality). Also, denote by $\mathcal{I}$ the collection of admissible sets where $f$ can be evaluated at. Then, $f$'s \emph{curvature} is defined as 
\begin{equation}\label{eq:curvature}
\nu_f\triangleq 1-\min_{\mathcal{S}\in\mathcal{I}}\min_{x\in\mathcal{S}}\frac{f(\mathcal{S})-f(\mathcal{S}\setminus\{x\})}{f({x})}.
\end{equation}
\end{definition}

The curvature, $\nu_f$, measures how far $f$ is from being additive. Particularly, Definition~\ref{def:curvature} implies $0 \leq \nu_f \leq 1$, and if $\nu_f=0$, then $f(\mathcal{S})=\sum_{x\in \mathcal{S}}f(\{x\})$  for all  $\mathcal{S}\in\mathcal{I}$ ($f$ is additive). On the other hand, if $\nu_f=1$, then there exist $\mathcal{S}\in\mathcal{I}$ and $x\in\mathcal{X}$ such that $f(\mathcal{S})=f(\mathcal{S}\setminus\{x\})$ ($x$ has no contribution in the presence of $\mathcal{S}\setminus\{x\}$).  

For example, in active target tracking, $f$ is the expected number of covered targets (as a function of the robot trajectories).  Then, $f$ has curvature 0 if each robot covers different targets from the rest of the robots. In contrast, it has curvature 1 if, e.g., two robots cover the exact same targets.

\subsection{Running time and approximation performance} \label{subsec:performance_analysis}

We present \alg's running time and suboptimality bounds.  To this end, we use the notation:
\begin{itemize}
\item $t_{\texttt{DCP}}^c$ and $t_{\texttt{DCP}}^s$ denote the communication and computation time of the robot that spends the highest time on three-round communications (Algorithm~\ref{alg:EDNCCA}, lines~\ref{line:encc_find_nei}, \ref{line:encc_share_rec_nei}, \ref{line:encc_share_rec_unique}) and neighbor set intersection (Algorithm~\ref{alg:EDNCCA}, line~\ref{line:encc_largest}) in \texttt{DCP};
\item $\mathcal{M}$ is a clique of $G$, which spends the highest time executing \crel;
\item $t_{\texttt{CRO}}^c$ denotes the communication time of the robot that spends the highest time exchanging information collected (e.g., the subsets of targets covered) in $\mathcal{M}$. 
\item $t^f$ denotes the time of one evaluation of the objective function $f$ (e.g., computing the number of targets covered by robots' actions). 
\item $\mathcal{X}_\mathcal{M}$ is the set of possible actions of all robots in $\mathcal{M}$; that is, $\mathcal{X}_\mathcal{M}\triangleq \cup_{i\in\mathcal{M}}\mathcal{X}_i$;
\item $f^\star$ is the optimal value of Problem~\ref{pro:dis_resi_sub};
\item $\mathcal{A}^\star(\mathcal{S})$ is a worst-case removal from $\mathcal{S}$ (a removal from $\mathcal{S}$ corresponds to a set of sensor attacks); that is,
$\mathcal{A}^\star(\mathcal{S})\in\arg\min_{\mathcal{A}\subseteq \mathcal{S}, |\mathcal{A}|\leq \alpha} \;f(\mathcal{S}\setminus \mathcal{A})$.

\end{itemize}

\begin{theorem}[Computational performance]~\label{thm:runtime}
\alg runs in $O(1) (t_{\texttt{DCP}}^c + t_{\texttt{DCP}}^s)$ + $O(1) t_{\texttt{CRO}}^c$ + $O(|\mathcal{X}_\mathcal{M}|^2) t^f$ time.  In addition, by \texttt{DRM}, each robot $i\in \mathcal{R}$ has four-round communications, including three rounds in \texttt{DCP} and one round in per clique \crel, and exchanges $3|\mathcal{N}_i|+ |\mathcal{C}_k|-1$ messages with $i\in \mathcal{C}_k$. Moreover, \texttt{DRM} performs $O(|\mathcal{N}_i|)$ operations for set intersections of each robot $i$ in \texttt{DCP} and $O(|\mathcal{X}_{\mathcal{C}_k}|^2)$ evaluations of objective function $f$ in \crel for each clique $\mathcal{C}_k, k \in \{1, \ldots, K\}$.    
\end{theorem}
$O(1) (t_{\texttt{DCP}}^c + t_{\texttt{DCP}}^s)$ corresponds to \alg's clique partition step (\alg's line~1), and  $O(1) t_{\texttt{CRO}}^c$ + $O(|\mathcal{X}_\mathcal{M}|^2) t^f$ corresponds to \alg's attack-robust step (\alg's lines 2-8). Particularly, \texttt{DRM}'s communication time includes the time of three-round communications in the clique partition step and the time of one-round communication of information collected in the attack-robust step; that is $O(1) t_{\texttt{DCP}}^c + O(1) t_{\texttt{CRO}}^c$. \texttt{DRM}'s computation time contains the time of neighbor set intersection in the clique partition step and the time of objective function evaluations in the attack-robust step; that is $t_{\texttt{DCP}}^s + O(|\mathcal{X}_\mathcal{M}|^2) t^f$. Notably, as the number of robots and/or the number of actions increase, \texttt{DRM}'s running time will be dominated by the time for function evaluations, i.e.,  $O(|\mathcal{X}_\mathcal{M}|^2) t^f$.


In contrast, to evaluate the objective functions, \crel~\cite[Algorithm~1]{zhou2018resilient} runs in $O(|\mathcal{X}|^2) t^f$ time. Thus, when $\mathcal{X}_\mathcal{M}\subset \mathcal{X}$ (which happens when $G$ is partitioned into at least $2$ cliques), \alg offers a computational speed-up.  The reasons are two: \textit{parallelization of action assignment}, and \textit{smaller clique size}. Particularly, \alg splits the action assignment among multiple cliques, instead of performing the assignment in a centralized way, where all robots form \textit{one large} clique (the $\mathcal{R}$).  That way, \alg enables each clique to work in \textit{parallel}, reducing the overall running time to that of clique $\mathcal{M}$ (Theorem~\ref{thm:runtime}).  
Besides parallelization, the smaller clique size also {contributes to the computational~reduction.}
To illustrate this, assume $G$ is partitioned to $K$ cliques of \textit{equal} size, and all robots have the same number of actions ($|\mathcal{X}_i|=|\mathcal{X}_j|$ for all $i,j\in\mathcal{R}$). Then, $O(|\mathcal{X}_\mathcal{M}|^2)=O(|\mathcal{X}|^2)/K^2$, that is, \alg's function evaluation time is reduced by the factor $K^2$. 

\begin{theorem}[Approximation performance of \alg]~\label{thm:DRA}
\alg returns a feasible action-set $\mathcal{S}$ such that
\begin{align}\label{eq:appro_kg2}
\begin{split}
\frac{f(\mathcal{S}\setminus \mathcal{A}^\star(\mathcal{S}))}{f^\star} \geq  \frac{1-\nu_f}{{2}}.
\end{split}
\end{align}
\end{theorem}

The proof of the theorem follows from~\cite[Theorem 1]{tzoumas2018resilient}.

From eq.~\eqref{eq:appro_kg2}, we conclude that even though \alg is a distributed, faster algorithm than its centralized counterpart, it still achieves a near-to-centralized performance. Generally, Theorem~\ref{thm:DRA} implies \alg guarantees a close-to-optimal value
for any submodular $f$ with curvature $\nu_f<1$.
\begin{remark}\label{rem:myopic}
A myopic algorithm that selects actions for each robot independently of all other robots (in contrast to \alg, whose subroutine \crel accounts for the other robots' actions during the greedy action assignment), guarantees the approximation bound $1-\nu_f$ ({Algorithm~\ref{alg:MM} in Appendix~C}). However, being exclusively myopic, {Algorithm~\ref{alg:MM}} has worse practical performance than $\texttt{DRM}$. 

{In Appendix C, we also show Algorithm~\ref{alg:MM} is equivalent to \crel (cf.~Section~\ref{subsec:DRA}) when applied to $\mathcal{R}$, under the assumption that \emph{all} robots in $\mathcal{R}$ are attacked. This further reveals the practical inefficiency of Algorithm~\ref{alg:MM}}. 
\end{remark}

\begin{remark}

Notably, \texttt{DCP} always finds a set of cliques. That is, it always terminates. However, in some degenerate cases, it may end up with a set of singletons, i.e., cliques that have one robot only. This issue can be mitigated by allowing for longer computation times. For example,  in \texttt{DCP}'s line~\ref{line:encc_largest}, there may exist multiple largest intersection sets $\mathcal{C}^{i}$ for each robot $i$. To break ties, we randomly select one, which may result in more singletons in the end. One natural way to find the largest intersection set that leads to the smallest number of singletons is to evaluate all the largest intersection sets $\mathcal{C}^{i}$.  Clearly, the evaluation takes more time. Specifically, the number of largest intersection sets for each robot $i$ is the number of its neighbors $|\mathcal{N}_i|$ in the worst case. Then evaluating all the largest intersection sets takes $O(N)$ time, while random selection takes $O(1)$ time. However, with longer computation time, \texttt{DCP} is more likely to generate fewer and larger cliques. 

The number of cliques plays a trade-off between the running time and practical performance of \texttt{DRM}. Specifically, for a given number of robots, \texttt{DRM} runs faster with more and smaller cliques since all cliques operate in parallel. While, at the same time, \texttt{DRM}'s  practical performance becomes poorer. This is because, with more and smaller cliques, more attacks will be inferred by \texttt{DRM} since each clique $\mathcal{C}_k$ infers it will incur $\alpha_k = \min\{\alpha, |\mathcal{C}_k|\}$ attacks. Therefore, \texttt{DRM} is more conservative and closer to the myopic algorithm introduced in Remark~\ref{rem:myopic}. On the contrary, with fewer and larger cliques, \texttt{DRM} gives a better practical performance but runs slower.

\end{remark}

\section{Improved Distributed Robust Maximization}\label{sec:improve_main}
In \texttt{DRM}, each clique $\mathcal{C}_k$ assumes that the number of attacks $\alpha_k$ against the clique is either equal to the total number of attacks $\alpha$, or equal to its size $|\mathcal{C}_k|$ (when $\alpha \geq |\mathcal{C}_k|$). Even though this strategy guarantees a close-to-optimal approximation performance (cf.~Section~\ref{sec:perform}), it is conservative, since the total number of attacks that all cliques infer ($\sum_{k} \alpha_k$) can be much larger than the real number of attacks $\alpha$ for the team. In this section, we  design a strategy to amend this conservativeness. Particularly, we present an \textit{Improved Distributed Robust Maximization} algorithm (\texttt{IDRM}), and analyze its performance in terms of approximation performance and running time. 

\subsection{{Improved inference of each clique's attack number}} 
The number of attacks in each clique $\alpha_k$ can be inferred by leveraging neighboring communications. First, note that robots can communicate with {all their neighbors within communication range} even though these neighbors are in different cliques. For example, in Fig.~\ref{fig:alg_step}, robot 2 can still communicate with robots 1 and 4 even though they are partitioned into different cliques. However, in \texttt{DRM}, this available communication is {ignored}. Second, besides the {3-hop} communications required for the execution of distributed clique partition (\texttt{DCP}; cf.~Algorithm~\ref{alg:EDNCCA}), the robots can also share their actions' function values (e.g., the number of targets covered) with their 3-hop neighbors.  Evidently, while \texttt{DRM} assumes robots to share actions' function values with their 1-hop neighbors within the same clique, sharing the actions' function values among 3-hop neighbors can give a better inference of $\alpha_k$, {and, consequently, better performance.} 

Recall that $\texttt{DRM}$ sets $\alpha_k$ equal to $\alpha$ (or to $|\mathcal{C}_k|$, when $\alpha \geq |\mathcal{C}_k|$) for each clique $\mathcal{C}_k$. Therefore, \alg selects a \emph{bait set of $\alpha_k$} robots in each clique (cf.~Section~\ref{subsec:DRA}). Evidently, some of these ``bait'' robots may not be among the \textit{$\alpha$ ``bait'' robots} that \crel would have selected if it would have been applied directly to $\mathcal{R}$ (assuming centralized communication among all robots).\footnote{{We refer to any robot in the selected bait set of a clique $k$ as a \emph{top $\alpha_k$ robot} in the clique; similarly, when we consider the set of all robots $\mathcal{R}$, the \emph{top $\alpha$ robots} are the robots in the bait set selected by a \crel applied to $\mathcal{R}$ (when the number of attacks against $\mathcal{R}$ is $\alpha$).}} This can be checked by communicating actions' function values among 3-hop neighbors.  Particularly, if some robot is not one of the top $\alpha$ robots among its 3-hop neighbors, it is impossible that it is in the top $\alpha$ robots among the entire team. Thus, this robot can be marked as ``unselected" and $\alpha_k$ can be reduced. Based on this rule, we describe our $\alpha_k$ inferring strategy in detail in Algorithm~\ref{alg:alpha_infer_known}.

\begin{algorithm}[t]
\caption{{Algorithm to approximate the number of attacks $\alpha_k$ against a clique $k$, given a known number of attacks $\alpha$ against all robots in $\mathcal{R}$.}}
\begin{algorithmic}[1]
\REQUIRE
    Robots' available actions $\mathcal{X}_i$,  $i\in \mathcal{R}$;
    monotone and submodular function $f$;
    attack number $\alpha$.
\ENSURE  Number of attacks $\hat{\alpha}_k$ for clique $\mathcal{C}_k$, ~$k \in \{1, \ldots, K\}$. 
\IF{$|\mathcal{C}_k| > \alpha$} \STATE $\hat{\alpha}_k = \alpha$; 
\ELSE 
\STATE $\hat{\alpha}_k = |\mathcal{C}_k|$;
\ENDIF
\STATE Each clique selects the top $\hat{\alpha}_k$ robots as $\mathcal{C}_k^{\hat{\alpha}_k}$;
\FOR{each robot $i \in \mathcal{C}_k^{\hat{\alpha}_k}$}
\IF{robot $i$ is not one of the top $\alpha$ robots in its 3-hop neighbors}
\STATE $\hat{\alpha}_k = \hat{\alpha}_k -1$;
\ENDIF
\label{afind: outsider}
\ENDFOR    
\RETURN $\hat{\alpha}_k$.  \label{afind:findbest}
\end{algorithmic}
\label{alg:alpha_infer_known}
\end{algorithm}
We use Fig.~\ref{fig:alg_step} to illustrate how Algorithm~\ref{alg:alpha_infer_known} works with an example. When $\alpha =2$, clique $\mathcal{C}_2:=\{2,8\}$ first infers $\hat{\alpha}_1 =2$, and robots 2 and 8 are selected. Then, robot 2 communicates with its 1-hop, 2-hop, and 3-hop neighbors (\{1, 4, 8\}, \{3, 5, 6, 7, 11\}, \{9, 10, 12, 13\}). {If robot 2 is not one of the top 2 robots among its 3-hop neighbors}, robot 2 will be marked as ``unselected" and $\hat{\alpha}_1$ will be reduced by 1 ($\hat{\alpha}_1 =1$). Similarly, if robot 8 is not one of the top 2 robots among its 3-hop neighbors, $\hat{\alpha}_1$ will be further reduced by 1 ($\hat{\alpha}_1 =0$). This way, instead of picking out 9 ``bait'' robots from 5 cliques (Fig.~\ref{fig:alg_step}-c), fewer robots will be selected. All in all, by using Algorithm~\ref{alg:alpha_infer_known}, we can reduce $\texttt{DRM}$'s conservativeness in inferring the number of attacks in each clique.

\subsection{Performance analysis of \texttt{IDRM}} \label{sec:performance-IDRM}
The robots of each clique $k$, using the number of attacks $\hat{\alpha}_k$ generated by Algorithm~\ref{alg:alpha_infer_known}, choose actions $\mathcal{S}_k$ by executing $\texttt{central-robust}$~\cite{zhou2018resilient}, that is, 
$$
\mathcal{S}_k= \crel\!\!~(\bigcup_{i\in \mathcal{C}_k}\mathcal{X}_i,f, \hat{\alpha}_k).
$$ 

\textbf{Approximation performance of \texttt{IDRM}.} \texttt{IDRM}  has the same approximation bound as that of \texttt{DRM} (cf.~eq.~\ref{eq:appro_kg2}). 
Notwithstanding, as we demonstrate in our numerical evaluations ({cf. Section~\ref{subsec:improve}}), \texttt{IDRM} performs better than \texttt{DRM} in practice, since \texttt{IDRM} utilizes more information (actions' function values shared among all 3-hop neighbors). When the communication graph $G$ only has non-overlapping cliques (i.e., the robots of each clique can communicate only with their neighbors within the clique), \texttt{IDRM} and \texttt{DRM} will exhibit the same performance. 

\textbf{Computational performance of \texttt{IDRM}.} \texttt{IDRM} runs in more time than \texttt{DRM}, since each robot needs to verify if it belongs to the top $\alpha$ robots among its 3-hop neighbors instead of its 1-hop neighbors within its clique as in \texttt{DRM}. But this verification only takes linear time. Thus, the running time of \texttt{IDRM} is similarly dominated by {the \texttt{central-robust} operating in all cliques} as in \texttt{DRM}. Also, each robot only needs to share with its 3-hop neighbors the function value of its best action instead of that of each action. Thus, \texttt{IDRM} keeps the computational advantage of \texttt{DRM}.
 
\section{Numerical Evaluations}~\label{sec:simulation}
We first present \alg's Gazebo and MATLAB evaluations in scenarios of \textit{active target tracking with swarms of robots}. Then we illustrate the advantages of \texttt{IDRM} by comparing it to \texttt{DRM}. 
The implementations' code is available online.\footnote{\url{https://github.com/raaslab/distributed_resilient_target_tracking.git}} We run the code on a ThinkPad laptop with Intel Core i7 CPU @ 2.6 GHz $\times$ 8  and  62.8 GB Memory by using Matlab 2018b and ROS Kinetic installed on Ubuntu 16.04. Due to the limited computer resources, we approximate the running time of {these distributed algorithms} by the total running time divided by the number of cliques, since all the cliques perform in parallel. Notably, the running time of the algorithms contains both the time of communication and the time of computation.

\subsection{Robust multi-robot target tracking}\label{subsec:sim_target_tracking} 

\textbf{Compared algorithms.} We compare \alg with two algorithms.  
First, the centralized counterpart of \alg in~\cite{zhou2018resilient}, named \crel (its near-optimal performance has been extensively demonstrated in~\cite{zhou2018resilient}). The second algorithm is the centralized greedy algorithm in~\cite{fisher1978analysis}, named \cgr. The difference between the two algorithms is that the former is attack-robust, whereas the latter is attack-agnostic.  For this reason, in~\cite{zhou2018resilient} we demonstrated, unsurprisingly, that \cgr has inferior performance to \crel in the presence of attacks.  However, we still include \cgr in the comparison, to highlight the differences among the algorithms both in running time and performance.
 
\subsubsection{Gazebo evaluation over multiple steps with mobile targets}
We use Gazebo simulations to evaluate \alg's performance across multiple rounds (time-steps). That way, we take into account the
kinematics and dynamics of the robots, as well as, the fact that the actual trajectories of the targets, along with the sensing noise, may force the robots to track fewer targets than expected. The motion model for the moving targets is known to the robots but it is corrupted with Gaussian noise. Therefore, the robots use a Kalman Filter to estimate the positions of the targets at each step. Due to the running efficacy of Gazebo (which is independent of \alg), we focus on small-scale scenarios of 10 robots. In the MATLAB simulation, we focus instead on larger-scale scenarios of up to 100 robots.

\textbf{Simulation setup.}
\begin{figure}[t]
\centering{
\subfigure[Gazebo environment]
{\includegraphics[width=0.46\columnwidth]{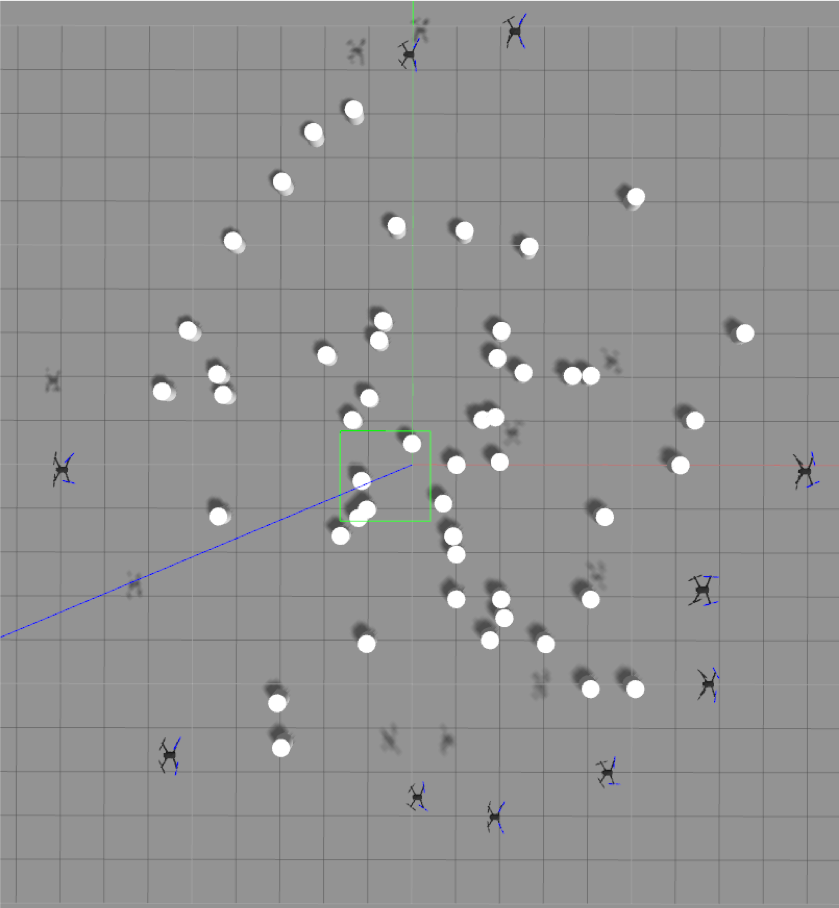}}~
\subfigure[Rviz environment] 
{\includegraphics[width=0.46\columnwidth]{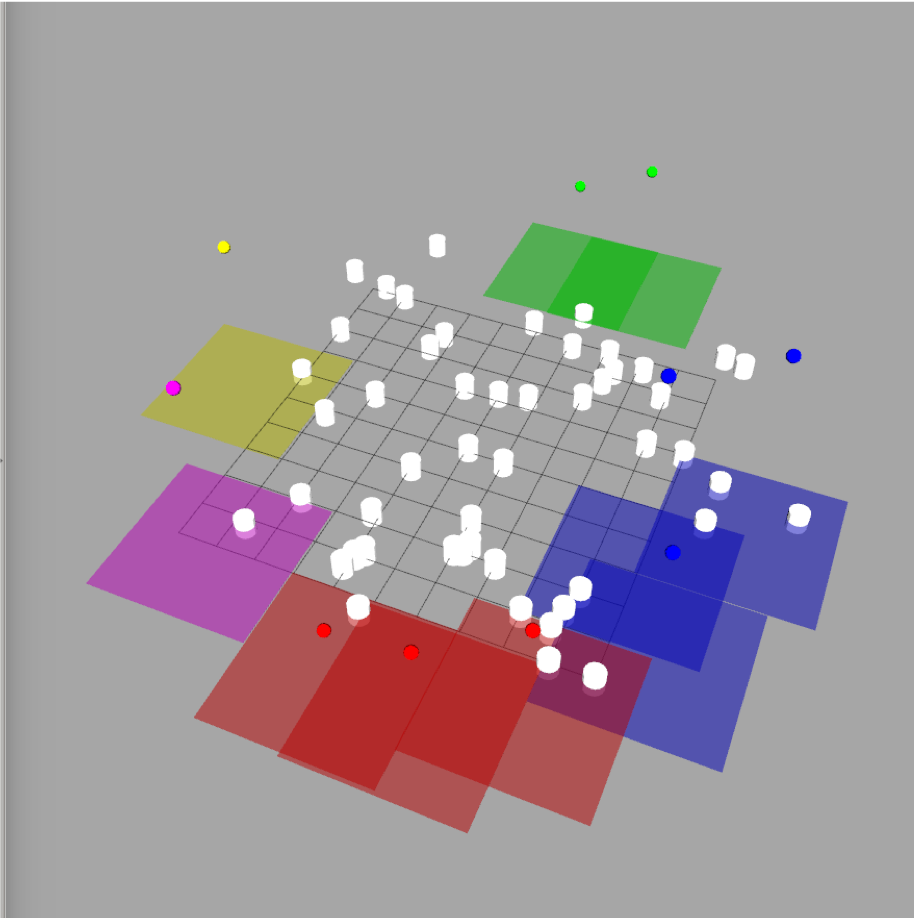}}
\caption{Gazebo simulation setup: 10 aerial robots and 50 ground mobile targets: (a) Gazebo environment; and (b) Rviz environment. Each robot is color-coded, along with its coverage region.  All robots in the same clique have the same color. The targets are depicted as white cylindrical markers. \label{fig:gazebo_rviz}}
}
\end{figure}
We consider 10 aerial robots that are tasked to track 50 ground mobile targets (Fig.~\ref{fig:gazebo_rviz}-(a)). We set the number of attacks $\alpha$ equal to $4$, and the robots' communication range to be $r_c = 5$ meters. We also visualize the robots, their field-of-view, their cliques, and the targets using the Rviz environment (Fig.~\ref{fig:gazebo_rviz}-(b)). Each robot $i$ has 5 candidates trajectories, $\mathcal{X}_i = \{\texttt{forward, backward, left,  right, stay}\}$, and flies on a different fixed plane (to avoid collision with other robots).
Each robot has a square field-of-view $l_o \times l_o$. Once a robot picks a non-stay trajectory, it flies a distance $l_f$ along that trajectory. If the robot selects the \texttt{stay} trajectory, it stays still (i.e., $l_f = 0$). Thus, each trajectory has a rectangular tracking region with length $l_t = l_f+l_o$ and width $l_o$. We set the tracking length $l_t = 6$, and tracking width $l_o=3$ for all robots. 
We assume robots obtain noisy position measurements of the targets, and then use a Kalman filter to estimate the target's position. We consider $f$ to be the expected number of targets covered, given all robots chosen trajectories. 

For each of the compared algorithms, at each round, each robot picks one of its 5 trajectories. Then, the robot flies a $l_f =3$ meters along the selected \texttt{non-stay} trajectory and stays still with the \texttt{stay} trajectory.  

When an attack happens, we assume the attacked robot's tracking sensor (e.g., camera) to be turned-off; nevertheless, we assume it can be active again at the next round. The attack is a worst-case attack, per Problem~\ref{pro:dis_resi_sub}'s framework.  Particularly, we compute the attack via a brute-force algorithm, which is viable for small-scale scenarios (as this one).

We repeat for 50 rounds. A video is available online.\footnote{\url{https://youtu.be/JA8IOStLndE}}

\textbf{Results.} 
\begin{figure}[th!]
\centering{
\subfigure[Running time ]{\includegraphics[width=0.50\columnwidth]{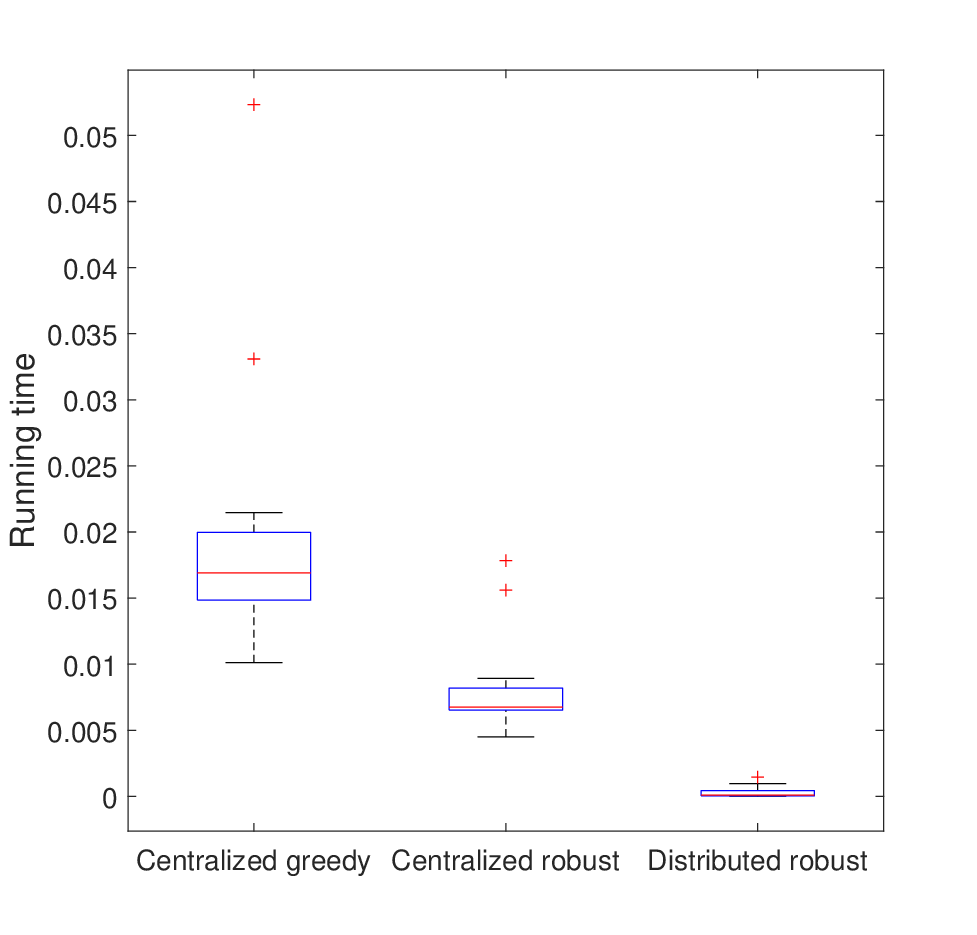}}  \hspace*{-2.5mm}
\subfigure[Number of targets tracked]{\includegraphics[width=0.50\columnwidth]{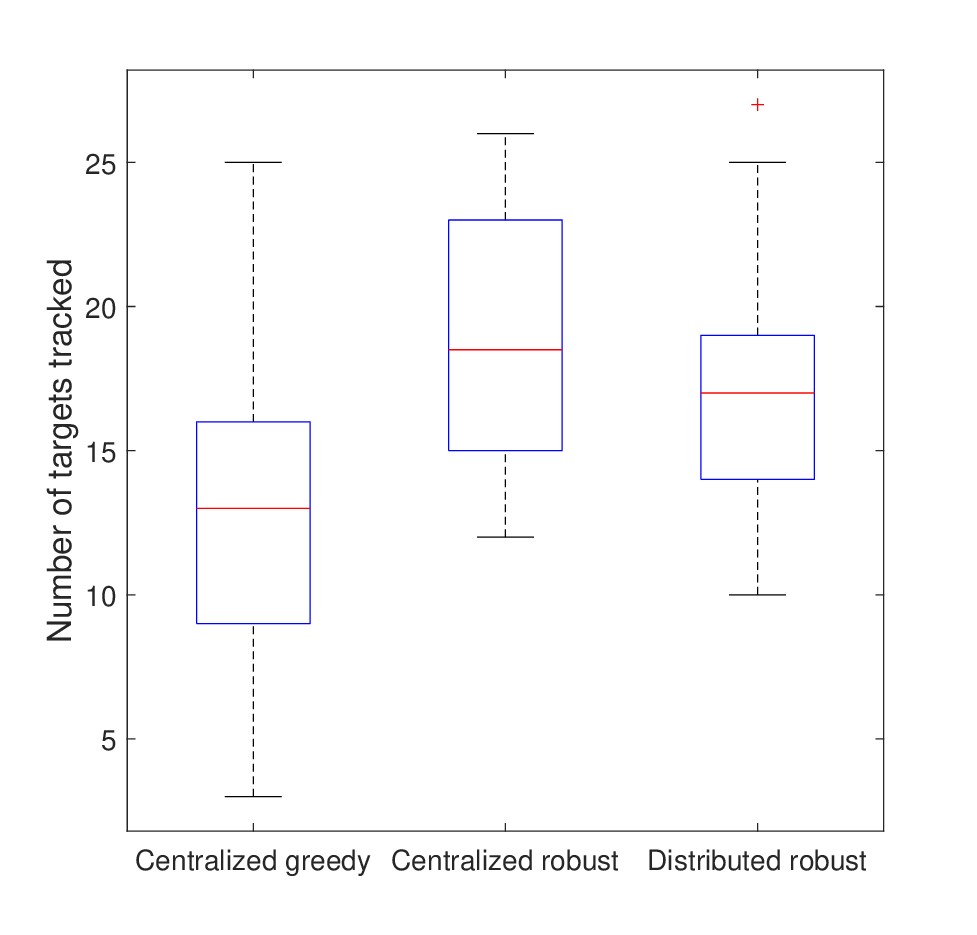}}
\caption{Gazebo evaluation (averaged across 50 rounds):
The tracking performance is captured by the number of covered targets per round. On each box in (a) and (b), the central red mark denotes the median, and the bottom and top edges of the box denote the 25th and 75th percentiles, respectively. The whiskers extend to the most extreme values not considered outliers, and the outliers are plotted individually using the red `+' symbol. If a box does not has outliers (as in (b)), the bottom and top ends of whiskers correspond to the minimum and maximum values.} 
\label{fig:compare_gazebo}}
\end{figure}
The results are reported in Fig.~\ref{fig:compare_gazebo}. 
We observe:

\textit{a) Superior running time:} \alg runs considerably faster than both \crel and \cgr: 1.8 orders faster than the former, and 2.2 orders faster than the latter, with average running time {0.1msec} (Fig.~\ref{fig:compare_gazebo}-(a)). 

\textit{b) Near-to-centralized tracking performance:} Despite that \alg runs considerably faster, it maintains near-to-centralized performance: \alg covers on average 18.8 targets per round, while \crel covers 17.1 (Fig.~\ref{fig:compare_gazebo}-(b)). To statistically evaluate \alg and \crel, we run a t-test with the default 5\% significance level on the targets covered by them. The t-test gives the test decision $H=0$ and $p$-value $p=0.0805$, which indicates that t-test does not reject the null hypothesis that the means of the number of targets covered by \alg and \crel are equal to each other at the 5\% significance level. This again demonstrates that  \alg performs very close to \crel. 

As expected, the attack-agnostic \cgr performs worse than all algorithms, even being centralized. Similarly, we run a t-test on the number of targets covered by \cgr and \alg and obtain $H=1$ and $p=0.00069$, which indicates that the means of the number of targets covered by \cgr and \alg are statistically different at the 5\% significance level.

\begin{figure*}[th!]
\centering{
\subfigure[10 robots with $r_c=120$]{\includegraphics[width=0.515\columnwidth]{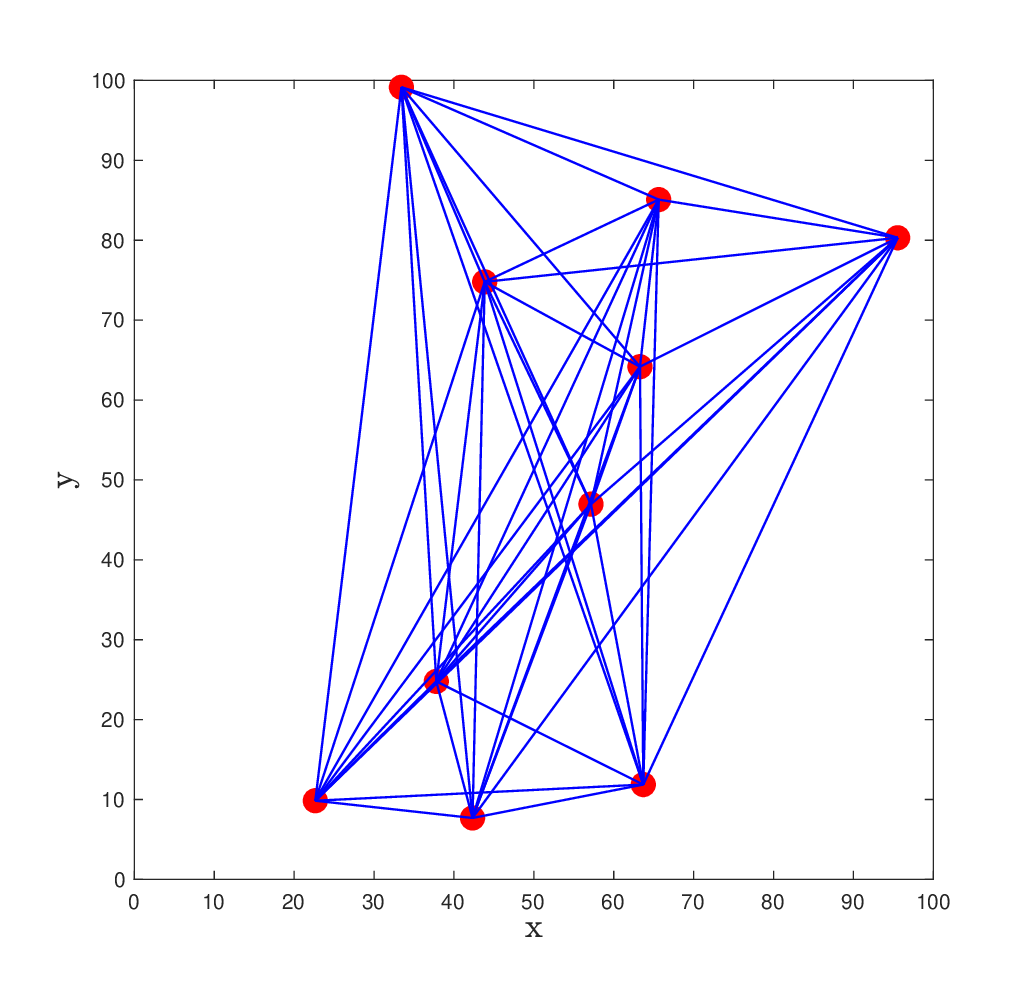}}~\subfigure[15 robots with $r_c = 60$ ]{\includegraphics[width=0.515\columnwidth]{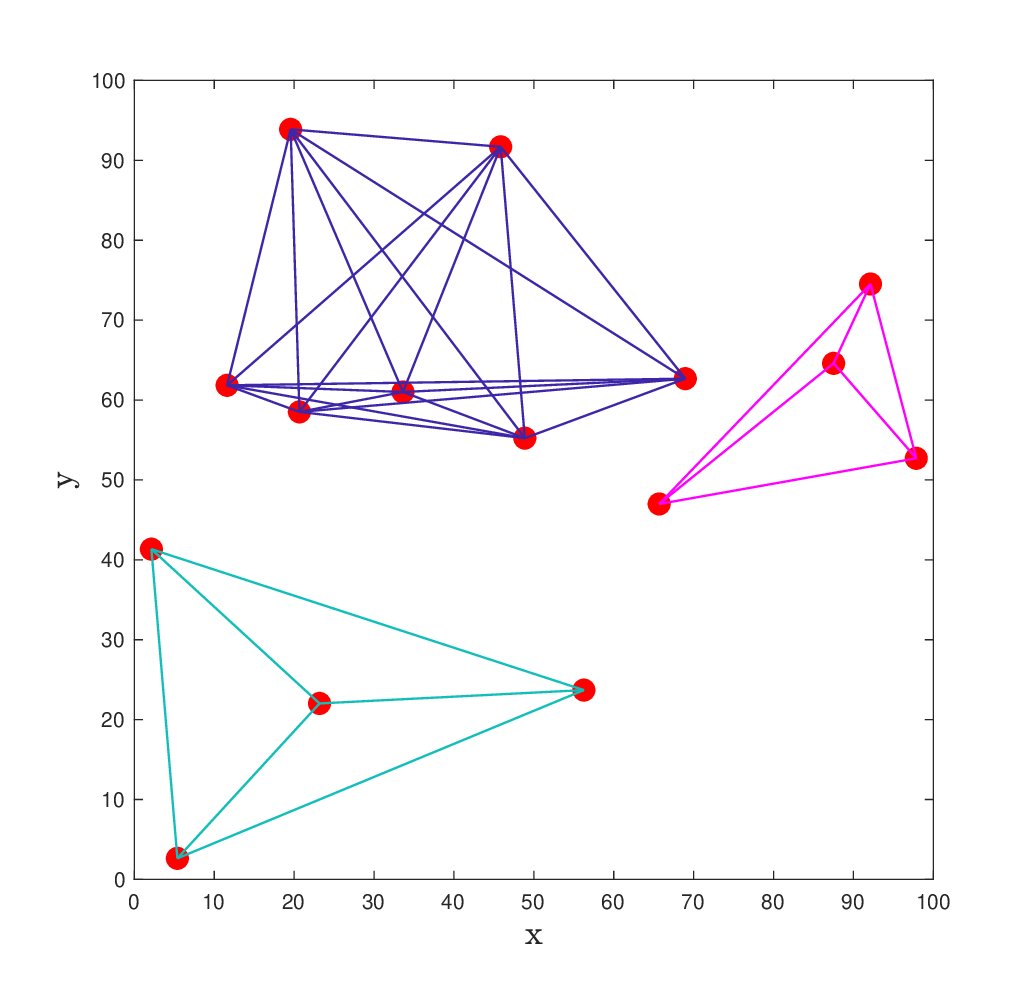}}~\subfigure[30 robots with $r_c = 90$] {\includegraphics[width=0.515\columnwidth]{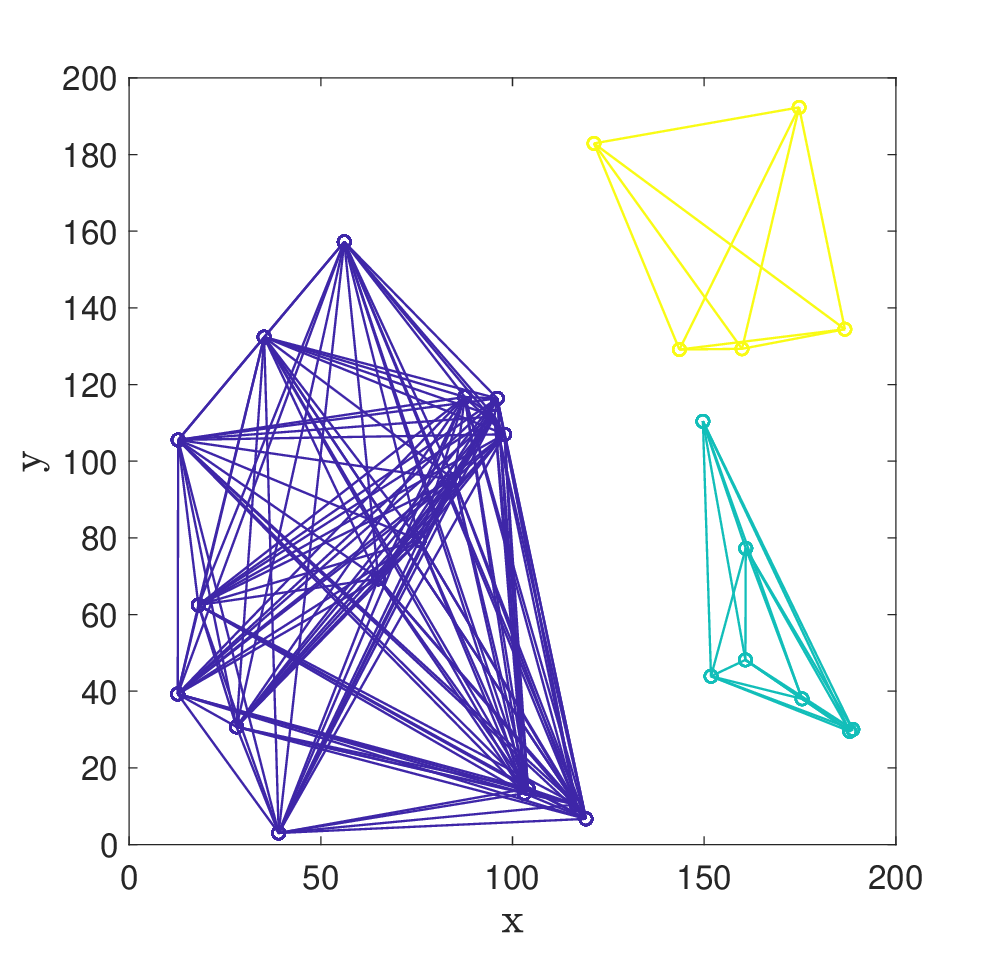}}~\subfigure[100 robots with $r_c = 50$]{\includegraphics[width=0.515\columnwidth]{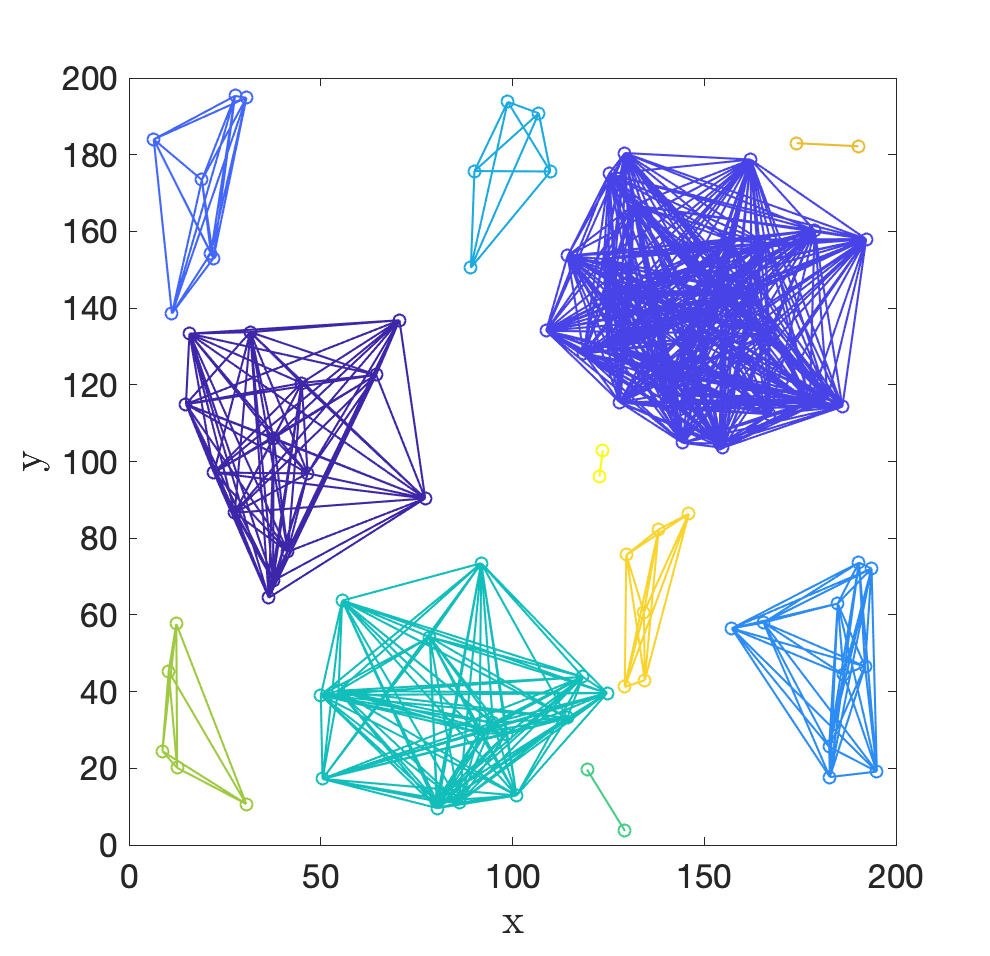}}
\caption{MATLAB evaluation: Examples of clique formulations { (Algorithm~\ref{alg:EDNCCA})} across various numbers of robots and communication ranges $r_c$. 
\label{fig:dis_non_over_cliq}}}
\end{figure*}

\subsubsection{MATLAB evaluation over one step with static targets}

We use MATLAB simulations to evaluate \alg's performance in large-scale scenarios. Specifically, we evaluate \alg's running time and performance for various numbers of robots (from 10 to 100) and communication ranges (resulting from as few as 5 cliques to as many as 30 cliques). We compare all algorithms over a single execution round.

\textbf{Simulation setup.} We consider $N$ mobile robots, and 100 targets. We vary $N$ from 10 to 100. For each $N$, we set the number of attacks equal to $\left \lfloor{N/4}\right \rfloor$, $\left \lfloor{N/2}\right \rfloor$ and $\left \lfloor{3N/4}\right \rfloor$. 

Similarly to the Gazebo simulations, each robot moves on a fixed plane, and has 5 possible trajectories: forward, backward, left, right, and stay. We set $l_t =10$ and $l_o = 3$ for all robots.  We randomly generate the positions of the robots and targets in a 2D space of size $[0, 200] \times [0,200]$.  Particularly, we generate 30 Monte Carlo runs (for each $N$). We assume that the robots have available estimates of targets' positions. 
For each Monte Carlo run, all compared algorithms are executed with the same initialization (same positions of robots and targets). \alg is tested across three communication ranges: $r_c = 30, 60, 90$. For a visualization of $r_c$'s effect on the formed cliques, see Fig.~\ref{fig:dis_non_over_cliq}, where we present two of the generated scenarios. All algorithms are executed for one round in each Monte Carlo run.

\begin{figure*}[th!]
\centering{
\subfigure[$r_c=30$, $\alpha =\left \lfloor{N/2}\right \rfloor$]{\includegraphics[width=0.515\columnwidth]{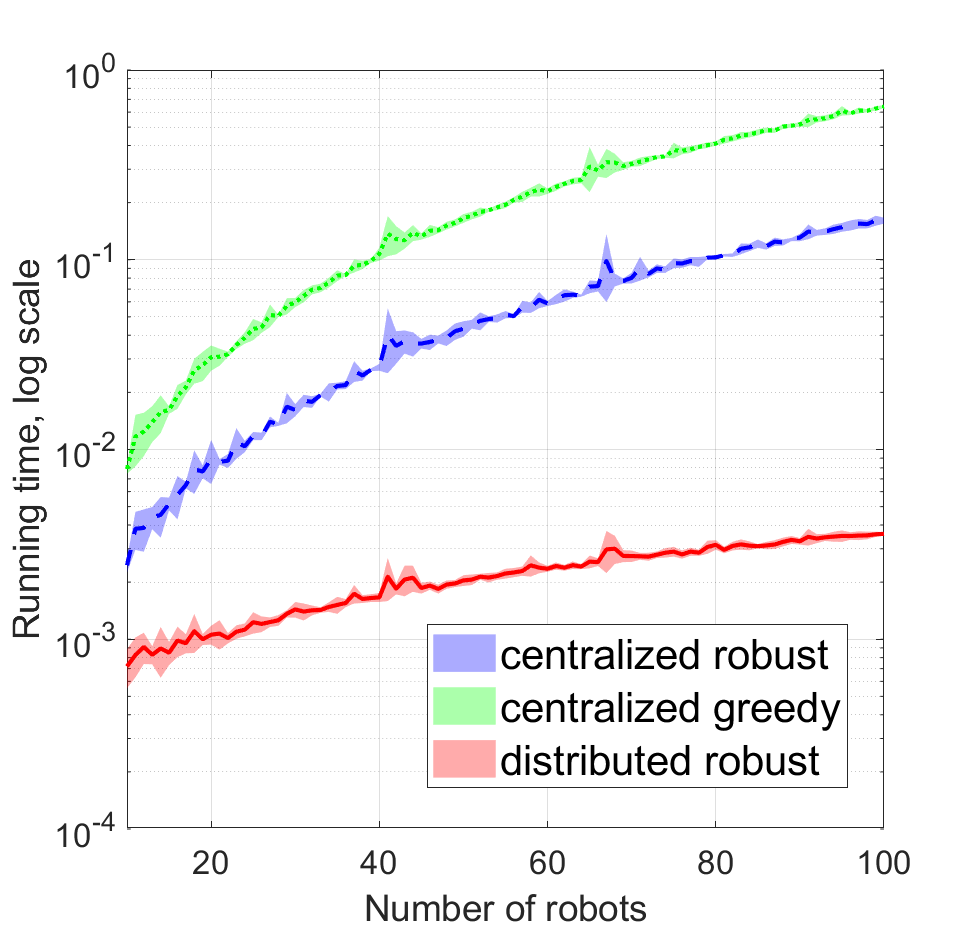}}~\subfigure[$r_c=90$, $\alpha =\left \lfloor{N/2}\right \rfloor$]{\includegraphics[width=0.515\columnwidth]{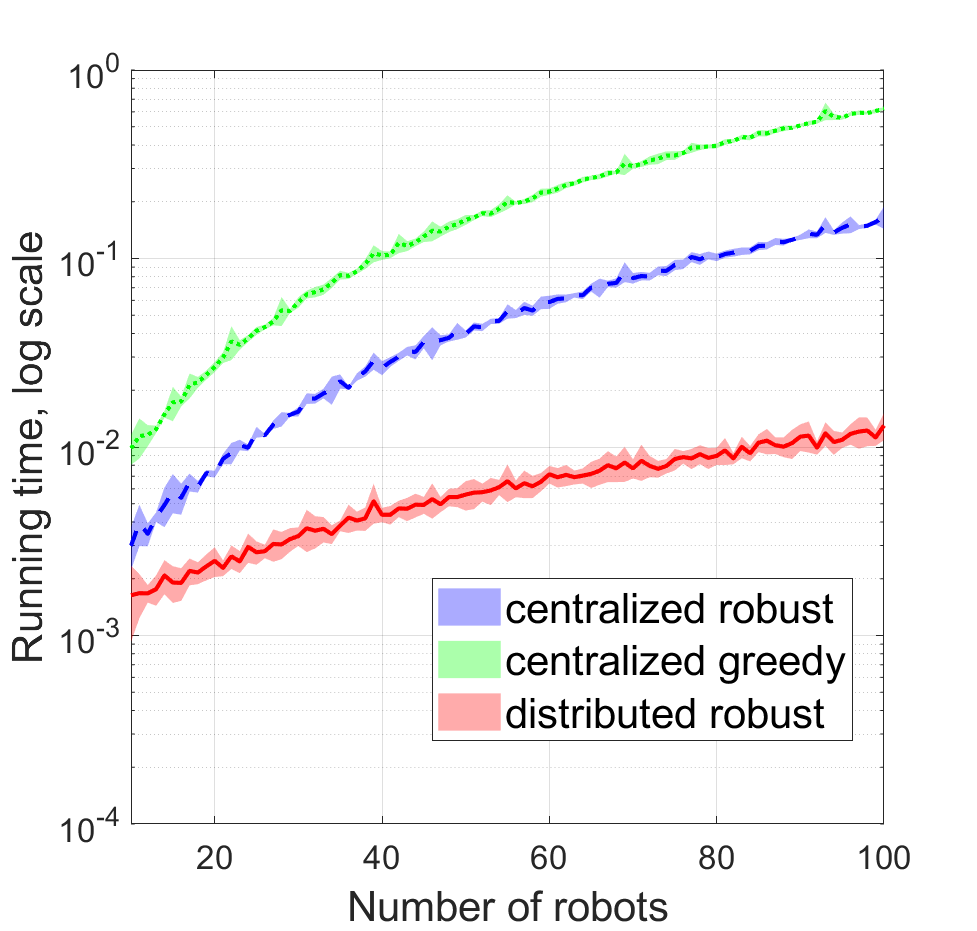}}~\subfigure[$\alpha =\left \lfloor{N/4}\right \rfloor$, $r_c=60$] {\includegraphics[width=0.515\columnwidth]{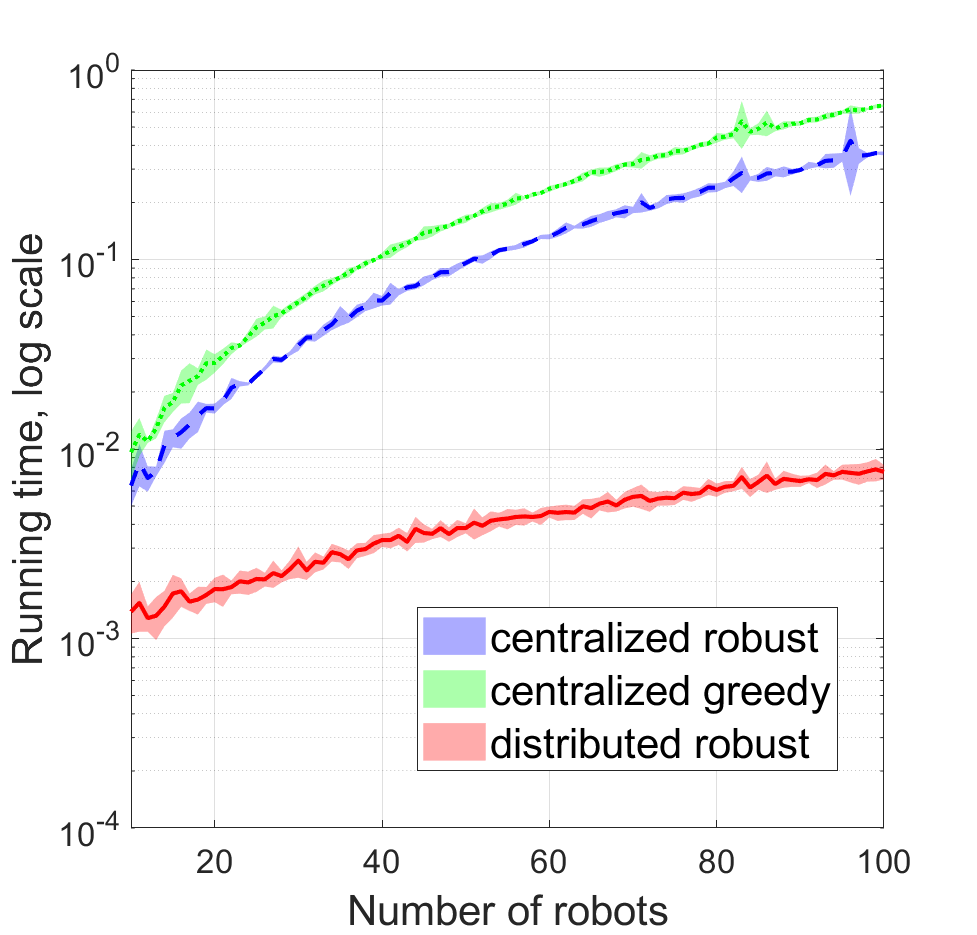}}~\subfigure[$\alpha =\left \lfloor{3N/4}\right \rfloor$, $r_c=60$ ]{\includegraphics[width=0.515\columnwidth]{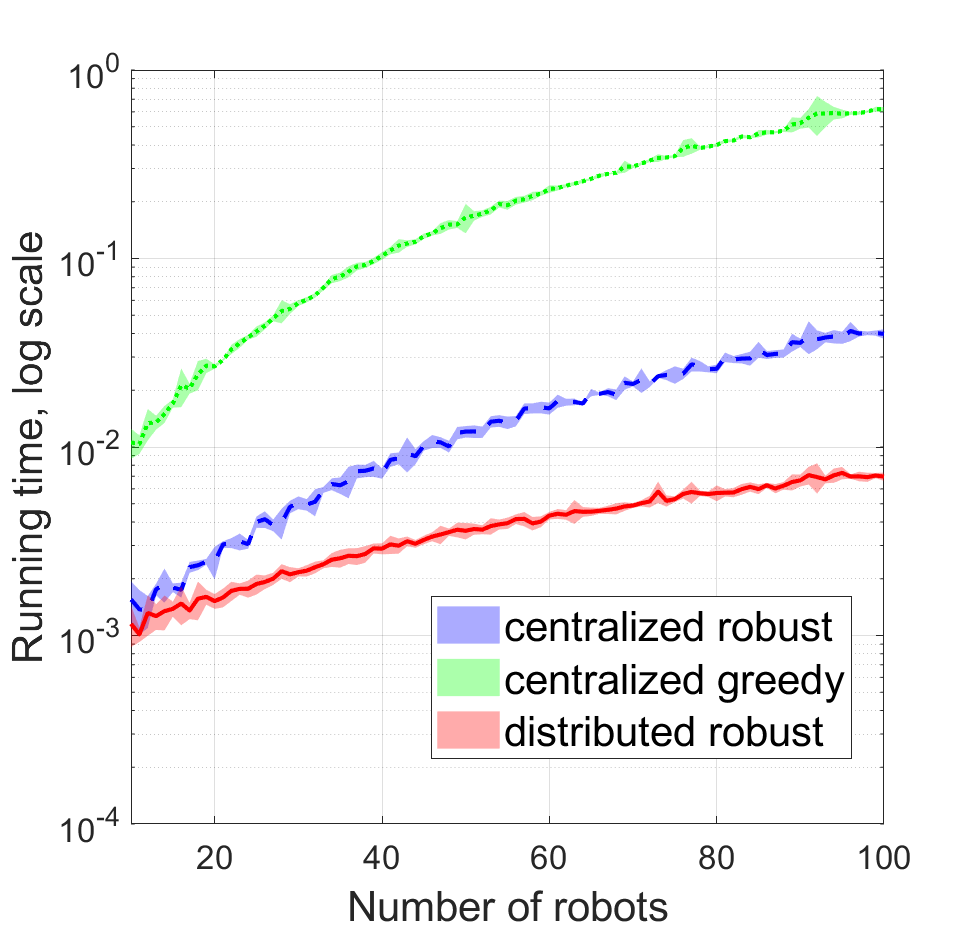}}\\

\subfigure[$r_c=30$, $\alpha =\left \lfloor{N/2}\right \rfloor$]{\includegraphics[width=0.515\columnwidth]{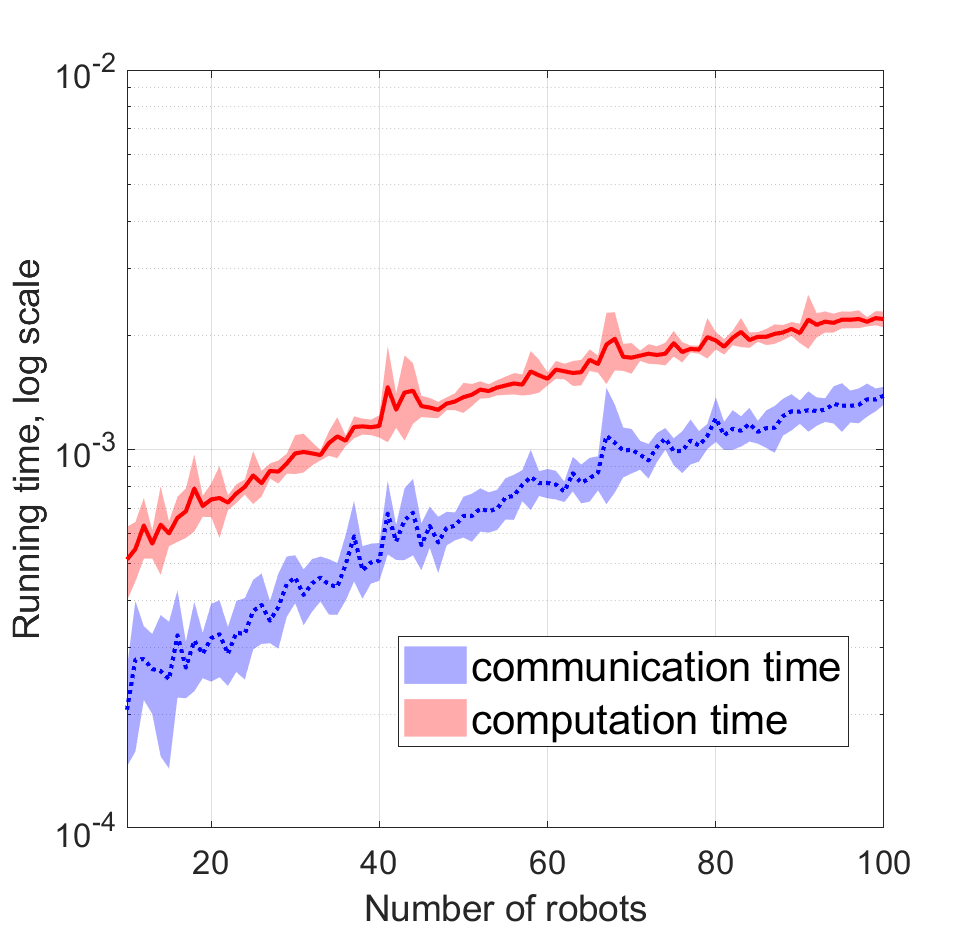}}~\subfigure[$r_c=90$, $\alpha =\left \lfloor{N/2}\right \rfloor$]{\includegraphics[width=0.515\columnwidth]{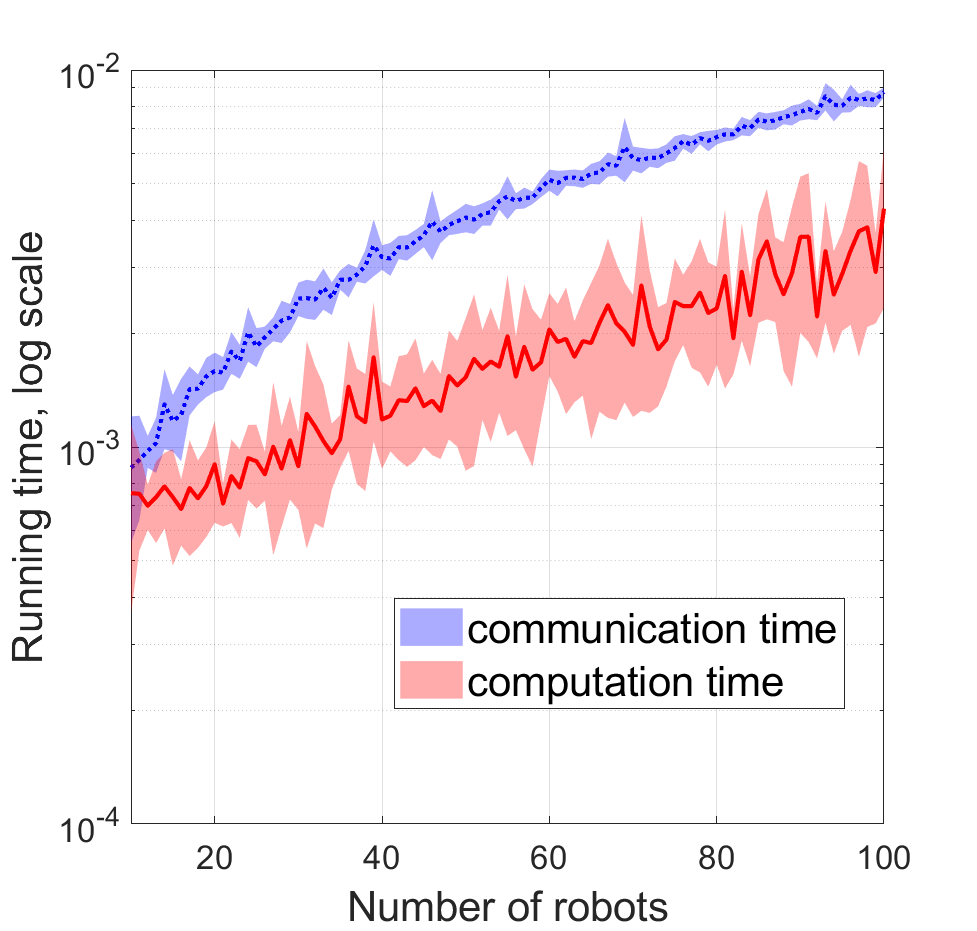}}~\subfigure[$\alpha =\left \lfloor{N/4}\right \rfloor$, $r_c=60$] {\includegraphics[width=0.515\columnwidth]{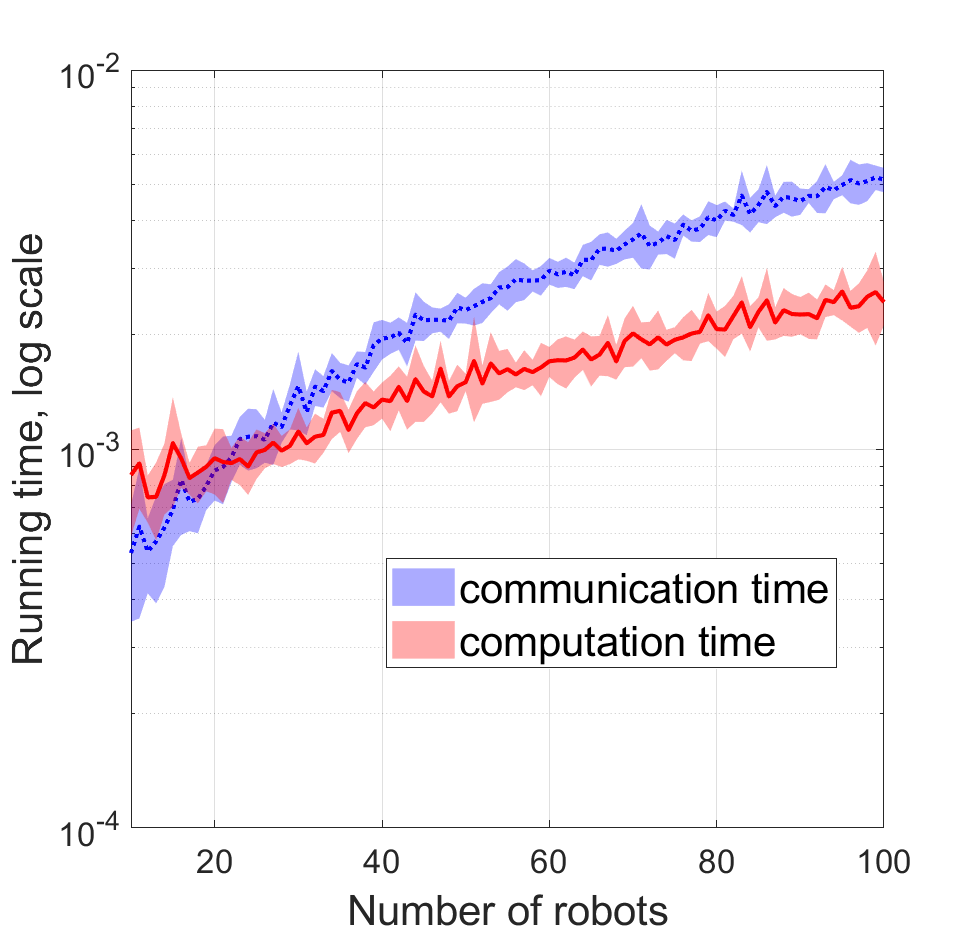}}~\subfigure[$\alpha =\left \lfloor{3N/4}\right \rfloor$, $r_c=60$ ]{\includegraphics[width=0.515\columnwidth]{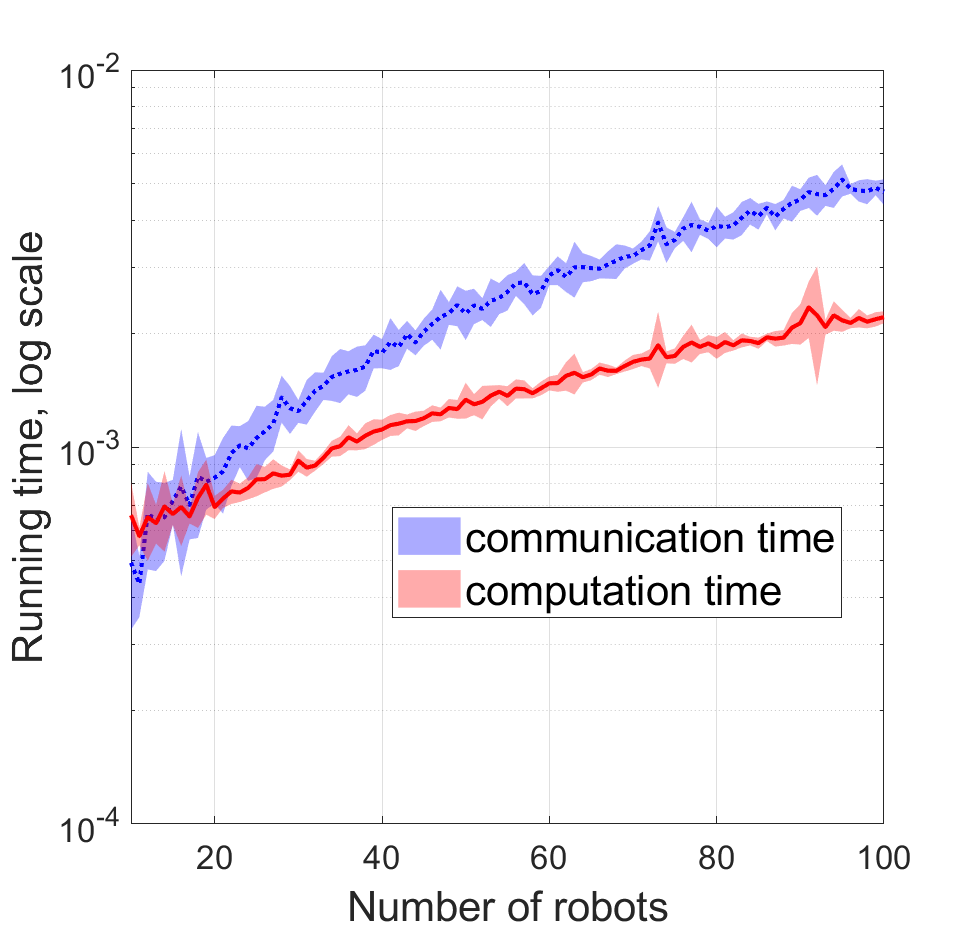}}\\

\subfigure[$r_c=30$, $\alpha =\left \lfloor{N/2}\right \rfloor$]{\includegraphics[width=0.515\columnwidth]{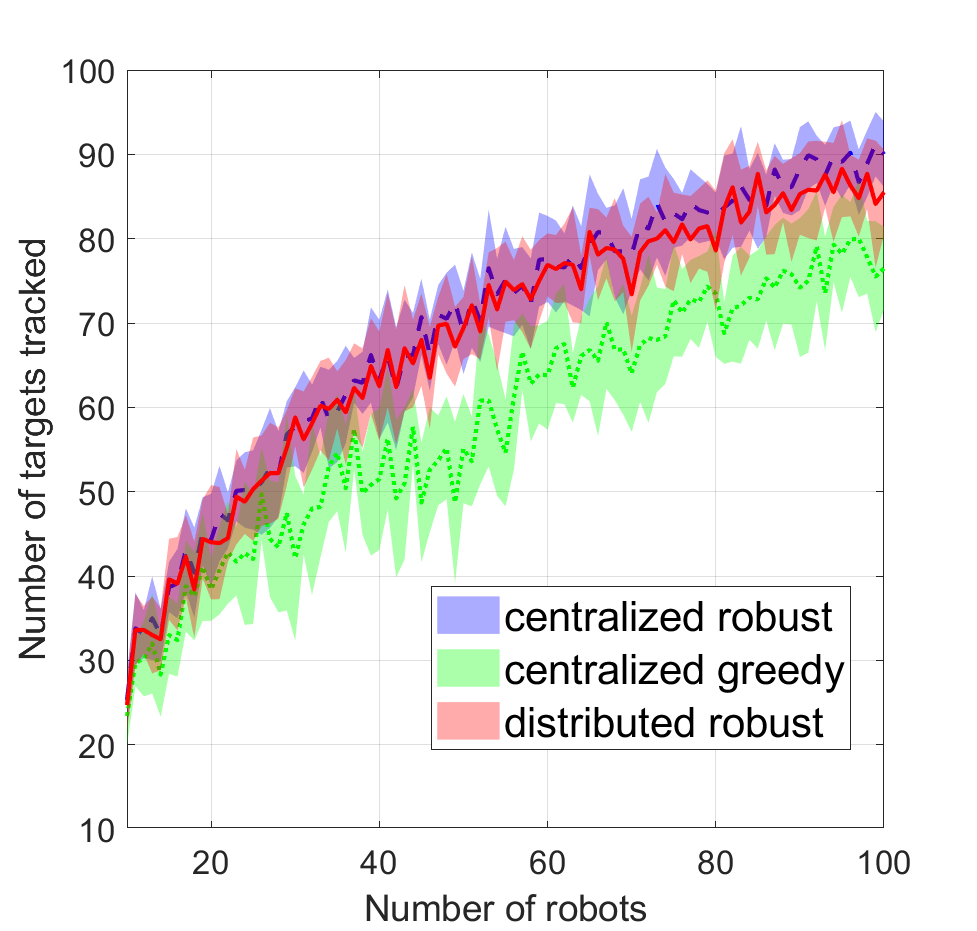}}~\subfigure[$r_c=90$, $\alpha =\left \lfloor{N/2}\right \rfloor$]{\includegraphics[width=0.515\columnwidth]{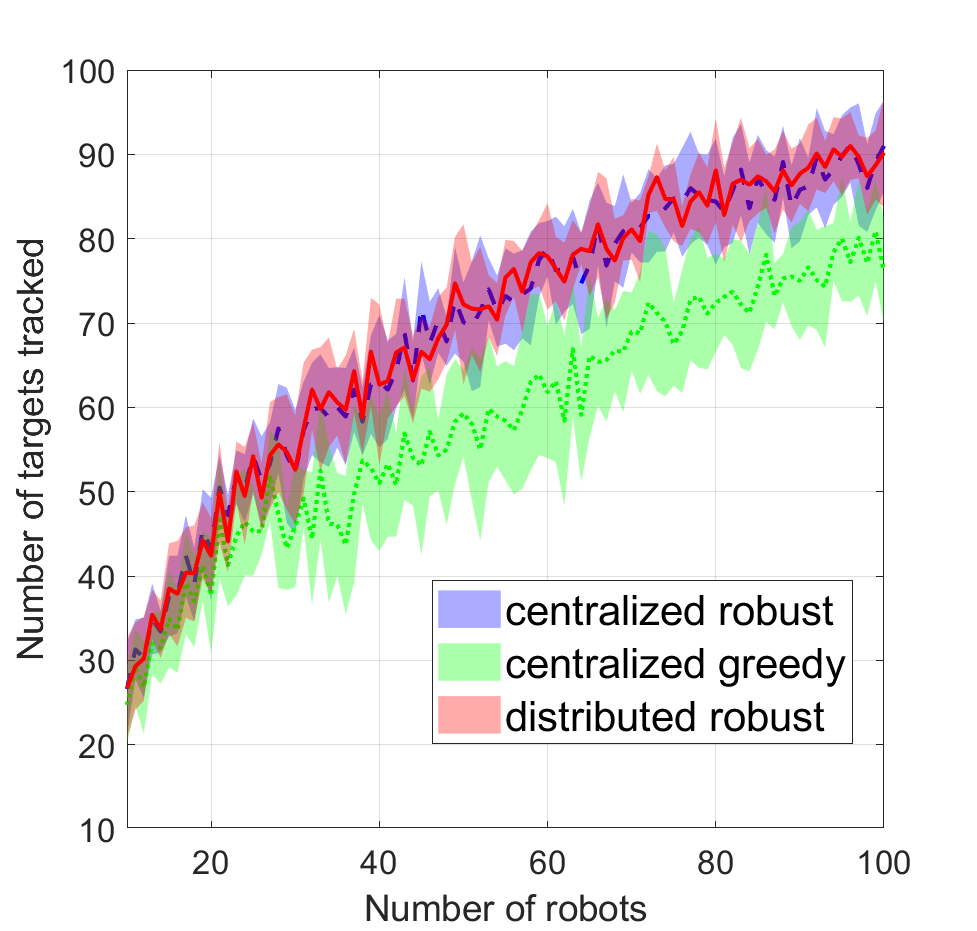}}~\subfigure[$\alpha =\left \lfloor{N/4}\right \rfloor$, $r_c=60$] {\includegraphics[width=0.515\columnwidth]{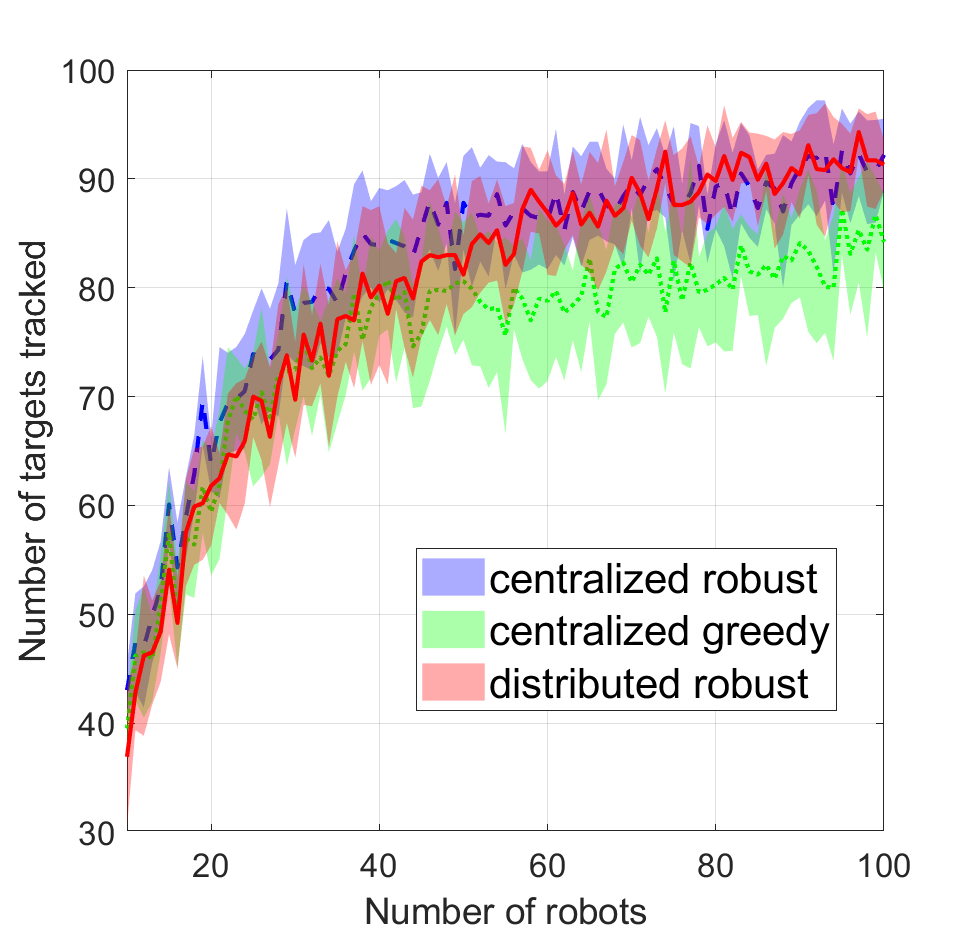}}~\subfigure[$\alpha =\left \lfloor{3N/4}\right \rfloor$, $r_c=60$]{\includegraphics[width=0.515\columnwidth]{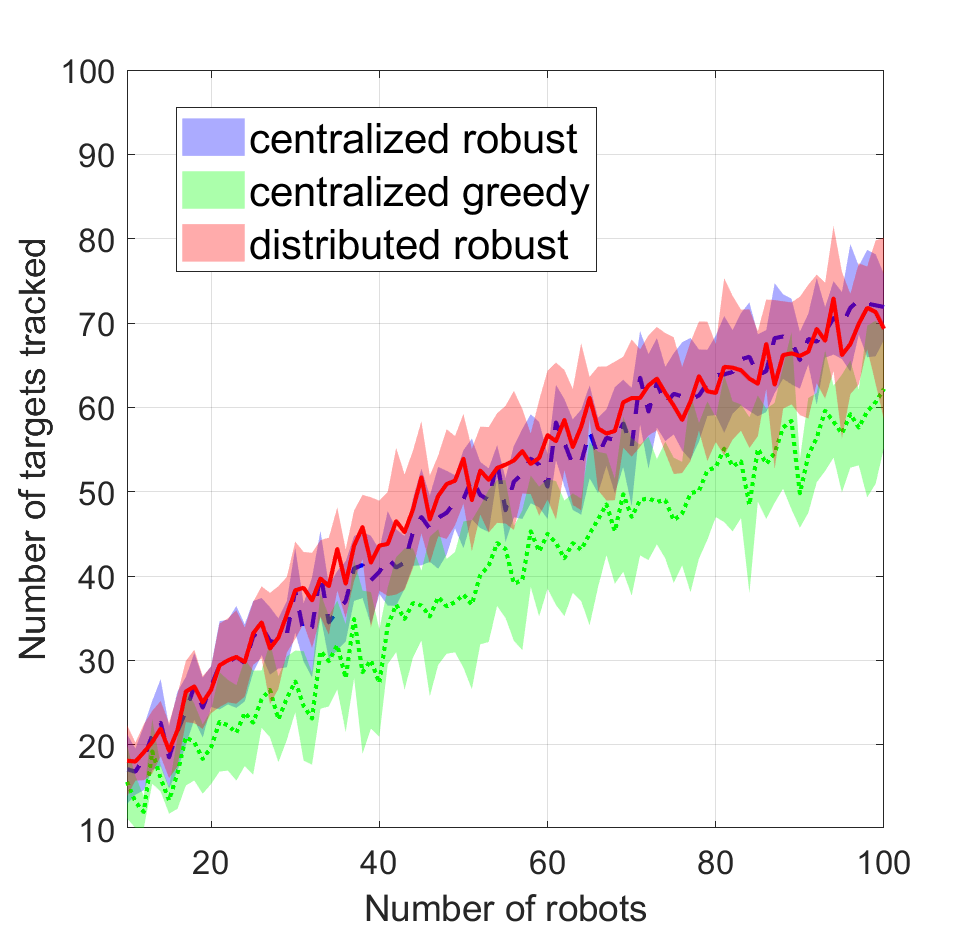}}
\caption{MATLAB evaluations (averaged across 30 Monte Carlo runs): (a)-(d) depict running time results of three algorithms, for various $\alpha$ and $r_c$ values; (e)-(h) depict the corresponding communication time and the computation time of \texttt{DRM} (\texttt{distributed-robust}); and (i)-(l) depict corresponding tracking performance results.
\label{fig:matlab_sim}}}
\end{figure*}

Notably, since we consider large-scale scenarios (up to $N=100$ robots, and up to 75 attacks, when $N=100$, and $\alpha=\lfloor{3N/4}\rfloor$), computing the worst-case attack via a brute-force algorithm is now infeasible. Particularly, given a trajectory assignment $\mathcal{S}$ to all robots, the problem of computing a worst-case attack is equivalent to minimizing a non-increasing submodular function,  an NP-hard problem~\cite{iyer2013fast}. Hence, we consider the attacker to use a greedy heuristic to attack the robots, instead of executing the worst-case attacks.  Proposed greedy approaches can be found in~\cite{iyer2013fast}.{\footnote{Alternative algorithms, along with approximate guarantees, to approximate the worst-case attacks can be found in \cite{topkis1978minimizing,schrijver2000combinatorial,jegelka2011fast}.}}

\textbf{Results.} The results are reported in Fig.~\ref{fig:matlab_sim} and we make the similar qualitative conclusions as in the Gazebo evaluation:

\textit{a) Superior running time:} \alg runs several orders faster than both \crel and \cgr: 1 to 2 orders on average, achieving running time from 0.5msec to 15msec (Figs.~\ref{fig:matlab_sim}-(a-d)). Particularly, in Figs.~\ref{fig:matlab_sim}-(a, b), the algorithms' running time is evaluated with respect to the communication range $r_c$ (with the number of attacks fixed as $\alpha =\left \lfloor{N/2}\right \rfloor$). We observe that \alg's running time increases as $r_c$ increases. This is because, with a larger $r_c$, \texttt{DCP} generates larger cliques and robots need to communicate more messages and takes more time on computation (Theorem~\ref{thm:runtime}), and thus both \texttt{DRM}'s communication time and computation time increase, which is shown in Figs.~\ref{fig:matlab_sim}-(e, f). Particularly, \texttt{DRM} spends more time on computation with a smaller $r_c$  (e.g., $r_c =30$) and spends more time on communication with a larger $r_c$ (e.g., $r_c = 90$).

In Figs.~\ref{fig:matlab_sim}-(c, d), the algorithms' running time is evaluated in terms of the number of attacks $\alpha$ (with the communication range fixed as $r_c = 60$).
We observe \crel  runs faster as $\alpha$ increases, which is due to how \crel works, that causes \crel to become faster as $\alpha$ tends to $N$~\cite{zhou2018resilient}. 
This parallels the observation that \alg's computation time decreases as $\alpha$ increases (Figs.~\ref{fig:matlab_sim}-(g-h)), since \alg's computation time mainly comes from the function evaluations in per clique \crel.

\textit{b) Near-to-centralized tracking performance:} Although \alg runs considerably faster, it retains a tracking performance close to the centralized one (Figs.~\ref{fig:matlab_sim}-(i-l)). 
Unsurprisingly, the attack-agnostic greedy performs worse than all algorithms. Particularly, as shown in Figs.~\ref{fig:matlab_sim}-(i, j) when the communication range $r_c$ increases from $30$ to $90$, \crel's tracking performance improves. This is because, with a larger $r_c$, fewer and larger cliques are generated by $\text{DCP}$, and \alg behaves closer to \crel. Figs.~\ref{fig:matlab_sim}-(k, l) show that all algorithms' tracking performance drops when the number of attacks $\alpha$ increases.

\smallskip

To summarize, in all simulations above, \alg offered significant computational speed-ups, and, yet, still achieved a tracking performance that matched the performance of the centralized, near-optimal algorithm in~\cite{zhou2018resilient}. 

\subsection{Improved multi-robot target tracking}~\label{subsec:improve}
\textbf{Compared algorithms.} We compare \texttt{IDRM} with \texttt{DRM}. The performance of the algorithms is evaluated through Matlab simulations on \textit{active target tracking} scenarios. We compare the algorithms in terms of the total number of attacks inferred, running time, and number of targets covered after $\alpha$ attacks. $\alpha$ is known to both \texttt{IDRM} with \texttt{DRM}. The algorithms are compared over a single execution for 30 trials.

\textbf{Simulation setup.} We consider $N$ mobile robots, and 100 targets. We set $N$ as 20 and 100 for small-scale and large-scale evaluations, respectively. For evaluating the algorithms in the large-scale case (e.g., $N=100$), we approximate the worst-case attack by a greedy attack since computing the worst-case attack requires exponential time. The total number of attacks $\alpha$ is set as 6 when $N=20$, and as 30 when $N=100$. The communication range $r_c$ is set as $r_c=120$ for $N=20$, and $r_c=70$ for $N=100$. The remaining settings are the same as in the Matlab simulation setup of Section~\ref{sec:simulation}-A.  

\begin{figure*}[t]
\centering{
\subfigure[Number of attacks inferred, $\alpha =6$]{\includegraphics[width=0.65\columnwidth]{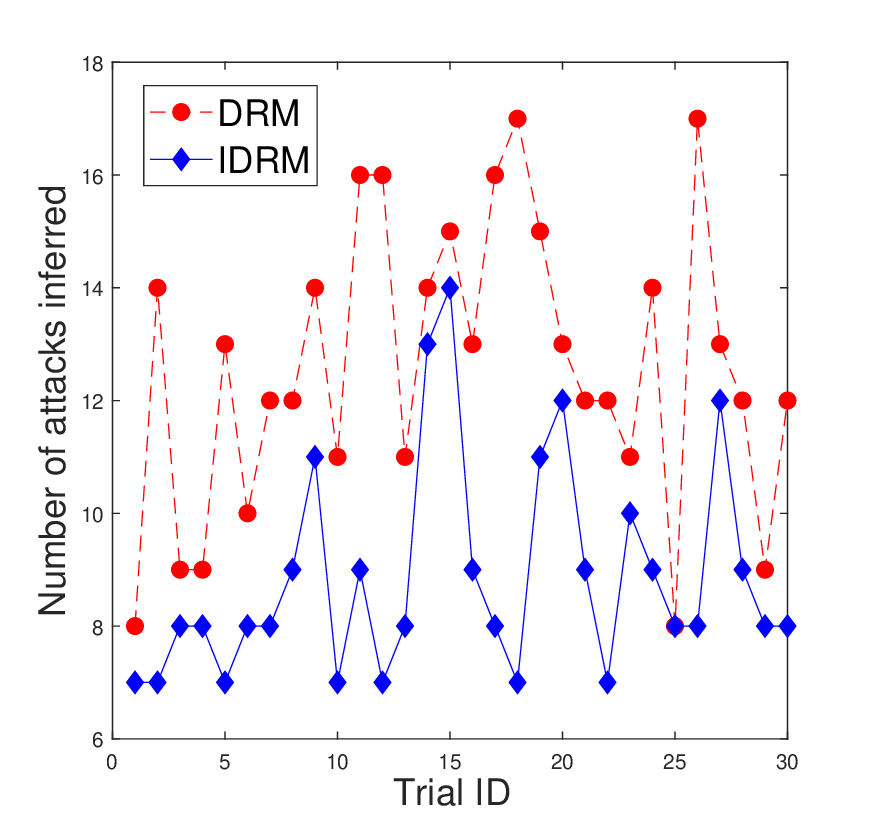}} ~~
\subfigure[Number of targets tracked]{\includegraphics[width=0.65\columnwidth]{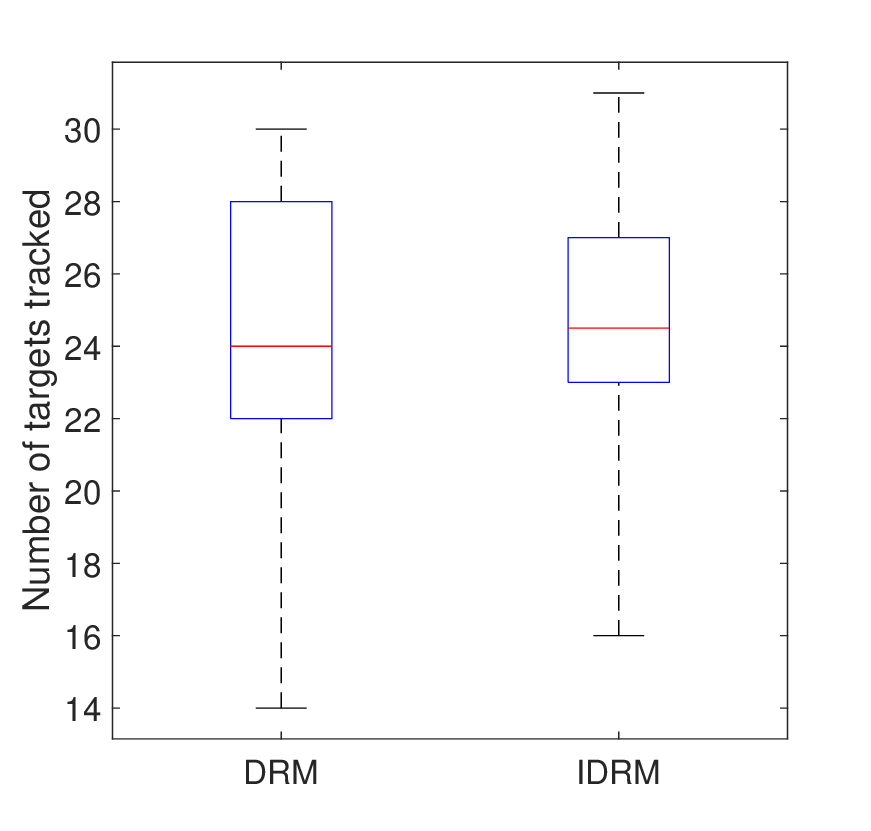}}~~
\subfigure[Running time]{\includegraphics[width=0.65\columnwidth]{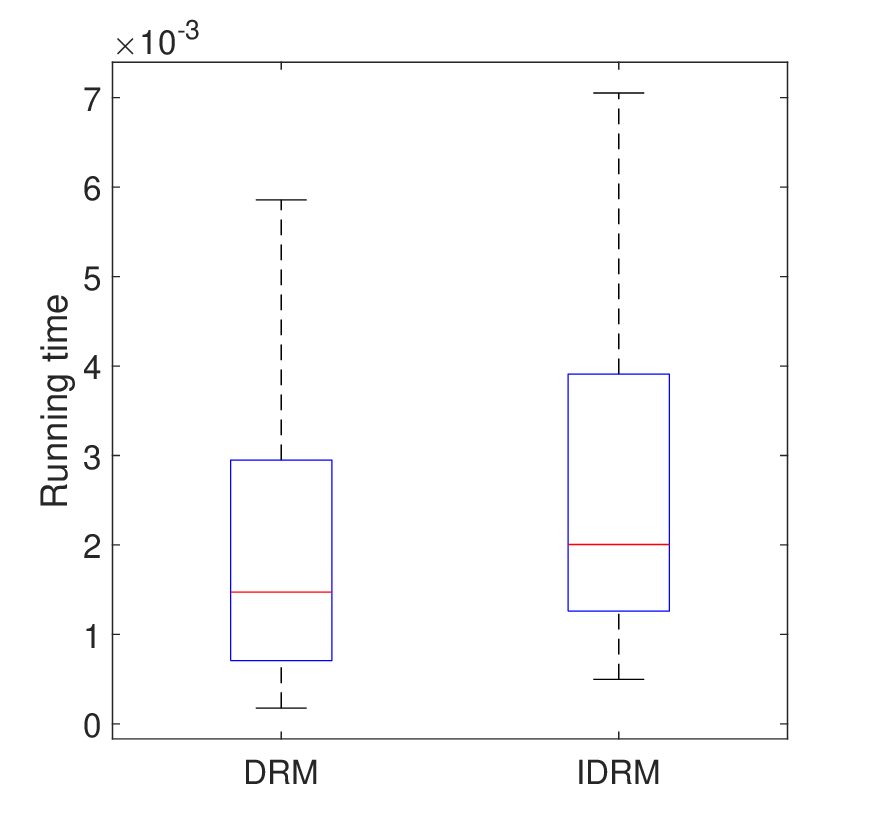}}
\caption{MATLAB evaluations with $N=20$, $\alpha=6$, and $r_c=120$: comparison of number of attacks inferred, number of targets covered, and running time for \texttt{DRM} and \texttt{IDRM} in small-scale case. The number of target tracked of these two algorithms are calculated after applying 6 worst-case attacks. In (b) and (c), the representations of each box's components follow the corresponding explanations in the caption of Fig.~\ref{fig:compare_gazebo}.
\label{fig:comp_drm_idrm_small}}}
\end{figure*}

\begin{figure*}[t]
\centering{
\subfigure[Number of attacks inferred, $\alpha=30$]{\includegraphics[width=0.65\columnwidth]{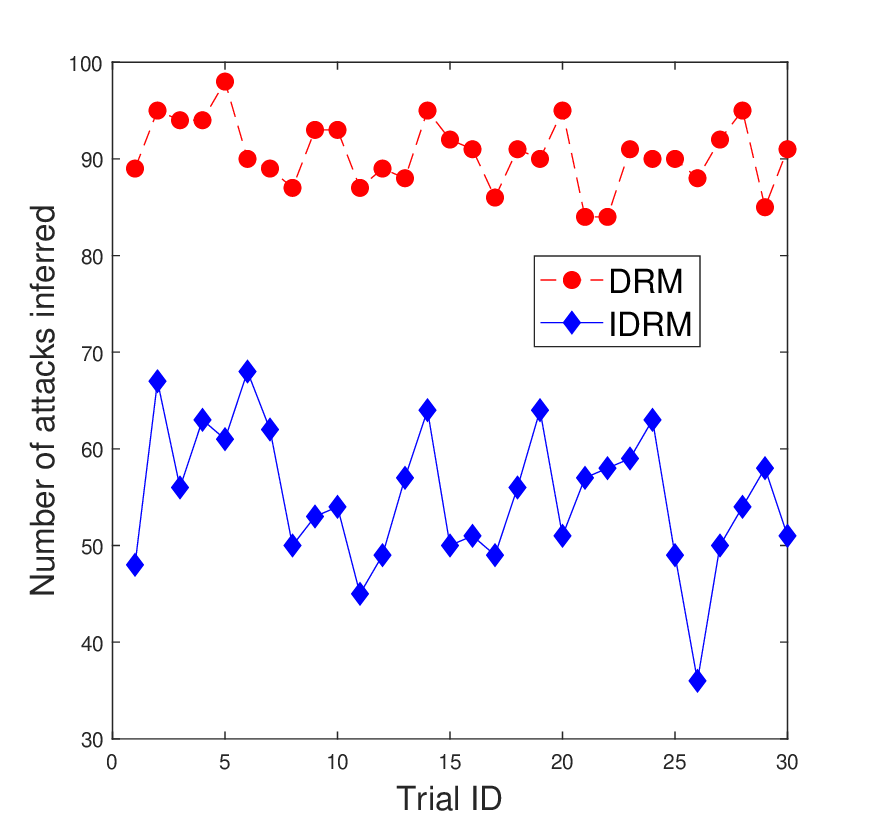}} ~~
\subfigure[Number of targets tracked]{\includegraphics[width=0.65\columnwidth]{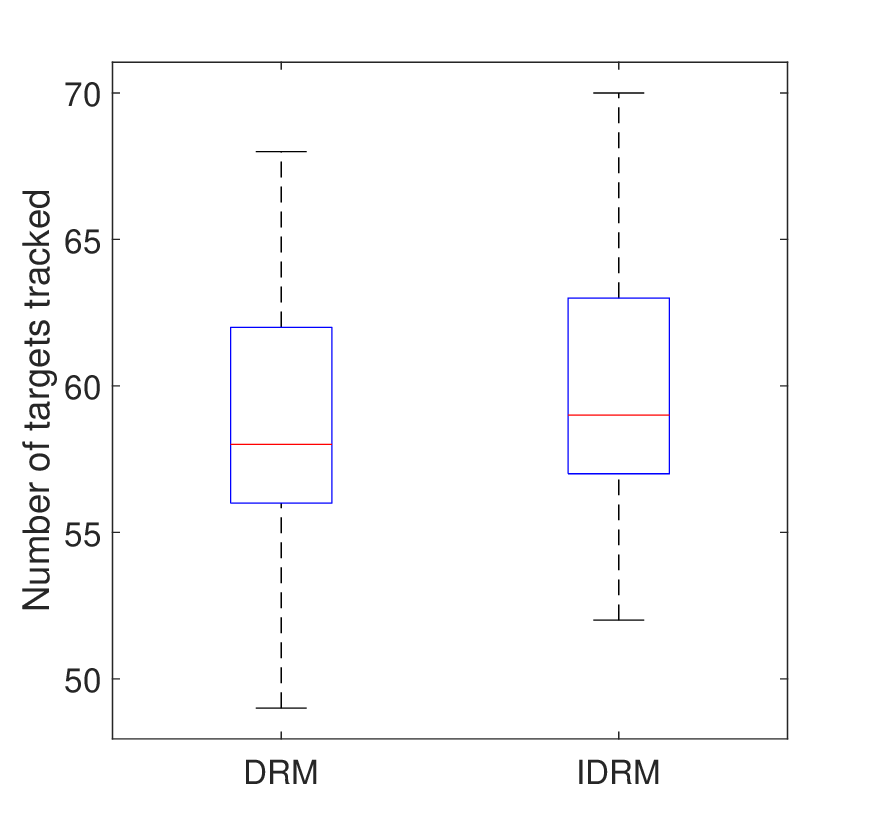}} ~~
\subfigure[Running time]{\includegraphics[width=0.65\columnwidth]{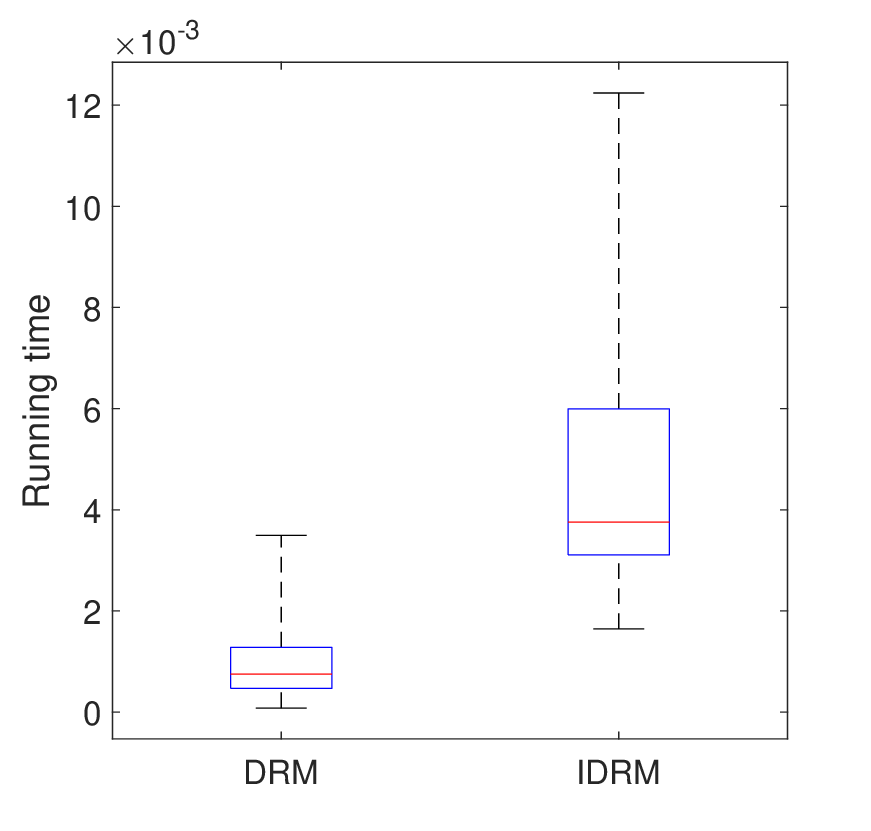}}
\caption{MATLAB evaluations with $N=100$, $\alpha=30$, and $r_c=70$: comparison of number of attacks inferred, number of targets covered, and running time for \texttt{DRM} and \texttt{IDRM} in large-scale case. The number of target tracked of these two algorithms are calculated after applying 30 greedy attacks. In (b) and (c), the representations of each box's components follow the corresponding explanations in the caption of Fig.~\ref{fig:compare_gazebo}.
\label{fig:comp_drm_idrm_large}}}
\end{figure*}

\textbf{Results.} The results are reported in Fig.~\ref{fig:comp_drm_idrm_small} and Fig.~\ref{fig:comp_drm_idrm_large}:

\textit{a) Conservativeness relaxing performance:} Fig.~\ref{fig:comp_drm_idrm_small}-(a) and Fig.~\ref{fig:comp_drm_idrm_large}-(a) show that \texttt{IDRM} relaxes the conservativeness of inferring number of attacks $\alpha$ in both small-scale ($N=20$, $\alpha=6$) and large-scale ($N=100$, $\alpha=30$) cases. Notably, when the communication range is smaller (e.g., $r_c=70$ in Fig.~\ref{fig:comp_drm_idrm_large}-(a)), the inferred number of attacks by \texttt{DRM} is much larger than the real number of attacks ($\alpha =30$). That is because, with a smaller communication range, the graph is likely to be partitioned into more and smaller cliques by Algorithm~\ref{alg:EDNCCA}, which increases the conservativeness of inferring $\alpha$ in \texttt{DRM}. While \texttt{IDRM} gracefully relaxes this conservativeness through 3-hop neighboring communications (Algorithm~\ref{alg:alpha_infer_known}).  Particularly, in some trials of both small-scale and large-scale evaluations, the number of attacks inferred by \texttt{IDRM} is very close to the real number of attacks $\alpha$. 

\textit{b) Superior tracking performance:} Because of the conservativeness relaxing, \texttt{IDRM} tracks more targets than \texttt{DRM} (Fig.~\ref{fig:comp_drm_idrm_small}-(b) and Fig.~\ref{fig:comp_drm_idrm_large}-(b)) since it reduces the unnecessary coverage overlaps induced by conservative estimate of $\alpha$. To further evaluate the tracking performance \texttt{IDRM} and \texttt{DRM}, we run a t-test with the default 5\% significance level on the targets covered by them. The t-test gives the test decision $H=1$ and $p$-value $p=0.025$ for the small-scale case (Fig.~\ref{fig:comp_drm_idrm_small}-(b)) and the test decision $H=1$ and $p$-value $p=0.030$ for the large-scale case (Fig.~\ref{fig:comp_drm_idrm_large}-(b)). This indicates that, in both cases, t-test rejects the null hypothesis that the means of the number of targets covered by \texttt{IDRM} and \texttt{DRM} are equal to each other at the 5\% significance level. Therefore, the difference between the means of the number of targets covered by them is statistically significant. 

\textit{c) Comparative running time:} Fig.~\ref{fig:comp_drm_idrm_small}-(c) and Fig.~\ref{fig:comp_drm_idrm_large}-(c) show that both \texttt{DRM} and \texttt{IDRM} run very fast (e.g., averaged running time is less than 0.005 seconds). That is because, in \texttt{IDRM}, after robots share the number of targets covered by their best actions and infer a less conservative $\alpha_k$ (Algorithm~\ref{alg:alpha_infer_known}), all cliques run \texttt{central-robust} in parallel as well. 

\section{Conclusion}\label{sec:conclusion}

We worked towards securing swarm-robotics applications against worst-case attacks resulting in robot withdrawals. Particularly, we proposed \alg, a distributed robust submodular optimization algorithm. \alg is general-purpose: it applies to any Problem~\ref{pro:dis_resi_sub}'s instance. We proved \alg runs faster than its centralized counterpart, without compromising approximation performance.  We demonstrated both its running time and  near-optimality in Gazebo and MATLAB simulations of {active target tracking}. {However, in \texttt{DRM}, each clique assumes the number of attacks $\alpha_k$ to be the total number of attacks $\alpha$, which can be too conservative if the robots are partitioned into many small-size cliques. To relax this conservativeness, we leveraged the 3-hop neighboring communications to present an improved version of \texttt{DRM}, called \texttt{IDRM}. We showed that \texttt{IDRM} improves the target-tracking performance of \texttt{DRM} with comparative running time. }

Note that, with \texttt{DCP} (Algorithm~\ref{alg:EDNCCA}), the clique partition solely depends on the positions of robots and the communication range. If robots have a very long communication range, they are likely to be all part of a single clique and then \texttt{DRM} would be the same as \crel. One way to address this issue is to intelligently manage the number of cliques to control the trade-off between complexity and optimality. Hence, one future research direction is to design such clique partition approaches that can be complementary to \texttt{DCP}. Another potential improvement on \texttt{DCP} is to relax its 3-hop communication constraints.  One way is to employ an anytime communication protocol that is tolerant to the failures of communication channels. Another direction is to enable asynchronous communication and achieve tolerance to communication delays.

 
 

A second future direction is to secure the team performance when the number of worst-case attacks is unknown. One heuristic approach to infer the number of attacks $\alpha_k$ distributively for cliques without knowing $\alpha$ is presented in the arXiv version \cite{zhou2019distributedatt} (cf. Algorithm 5 in Appendix-D of \cite{zhou2019distributedatt}). 
However, this algorithm is a heuristic approach that aims to secure the \textit{expected} worst-case performance and thus may not have guarantees against a specific number of worst-case attacks. Therefore, our future work is focused on designing algorithms that have approximation guarantees against the unknown number of attacks. A third future avenue is to investigate other attack or failure models, e.g., random failures~\cite{park2018robust,zhou2018approximation}, and design corresponding distributed robust algorithms.  

\section*{Acknowledgments}

We thank Micah Corah from the Carnegie Mellon University for pointing out that the myopic Algorithm~\ref{alg:MM}, stated in Appendix~C, achieves the performance bound $1-\nu_f$ (Theorem~\ref{thm:mm} in Appendix~C).  The observation led to Remark~\ref{rem:myopic}.

\section*{Appendix}

\subsection{Proof of Theorem~\ref{thm:runtime}}
\begin{proof}
\alg's running time is equal to \texttt{DCP}'s plus the time for all cliques to execute \crel in parallel.

    Particularly, the running time of \texttt{DCP} contains the time of three-round communications---finding neighbors (Algorithm~\ref{alg:EDNCCA}, line~\ref{line:encc_find_nei}), exchanging neighbor sets (Algorithm~\ref{alg:EDNCCA}, line~\ref{line:encc_share_rec_nei}) and exchanging computed cliques (Algorithm~\ref{alg:EDNCCA}, line ~\ref{line:encc_share_rec_unique}) and the time of neighbor set intersection (Algorithm~\ref{alg:EDNCCA}, line~\ref{line:encc_largest}). Since robots perform in parallel by \texttt{DCP}, \texttt{DCP}'s running time depends on the robot that spends the highest time on three-round communications and neighbor set intersection. Thus, \texttt{DCP} takes $O(1) (t_{\texttt{DCP}}^c + t_{\texttt{DCP}}^s)$ time. Specifically, in each communication round, each robot $i$ sends its message to and receives the messages from its neighbors, and thus the number of messages exchanged by each robot $i$ is $|\mathcal{N}_i|$. Hence, over three-round communications, each robot exchanges $3|\mathcal{N}_i|$ messages. During the neighbor set intersection, each robot $i$ intersects its neighbor set with those of its neighbors, and thus the number of set intersections (operations) for each robot $i$ is $|\mathcal{N}_i|$.

Next, all cliques perform \crel~\cite{zhou2018resilient} in parallel. In each clique $\mathcal{C}_k$, each robot $i$ first communicates and exchanges information collected (e.g., the subsets of targets covered by its candidate actions) with all the other robots. Hence, the number of messages exchanged by each robot $i$ in $\mathcal{C}_k$ is $|\mathcal{C}_k|-1$. Notably, since robots can communicate with each other in $\mathcal{C}_k$, the communication only takes one round. 

After information exchange, the robots in $\mathcal{C}_k$ perform \crel that executes sequentially two steps---the bait assignment and the greedy assignment. The bait assignment sorts out the best $\min(\alpha, |\mathcal{C}_k|)$ single actions from $|\mathcal{C}_k|$ robots. We denote the joint candidate action set in $\mathcal{C}_k$ as  $\mathcal{X}_{\mathcal{C}_k} \triangleq \bigcup_{i\in \mathcal{C}_k} \mathcal{X}_i$. Then, the sorting needs $O(|\mathcal{X}_{\mathcal{C}_k}| \log (|\mathcal{X}_{\mathcal{C}_k}|)$ evaluations of objective function $f$ and thus takes $O(|\mathcal{X}_{\mathcal{C}_k}| \log (|\mathcal{X}_{\mathcal{C}_k}|)) t^f$ time. Then, the greedy assignment uses the standard greedy algorithm~\cite{fisher1978analysis} to choose actions for the remaining robots, which needs $O(|\mathcal{X}_{\mathcal{C}_k}|^2)$ evaluations of $f$ and thus takes $O(|\mathcal{X}_{\mathcal{C}_k}|^2) t^f$ time. Hence, the number of function evaluations (operations) is $O(|\mathcal{X}_{\mathcal{C}_k}| \log (|\mathcal{X}_{\mathcal{C}_k}|)$ + $O(|\mathcal{X}_{\mathcal{C}_k}|^2)$ = $O(|\mathcal{X}_{\mathcal{C}_k}|^2)$. 

Since all cliques perform \crel in parallel, the running time depends on the clique that spends the highest time, i.e., clique $\mathcal{M}$. $\mathcal{M}$'s running time contains the time of exchanging information collected, i.e., $O(1) t_{\texttt{CRO}}^c$, and of evaluating function $f$, i.e., $O(|\mathcal{X}_{\mathcal{C}_k}| \log (|\mathcal{X}_{\mathcal{C}_k}|) t^f + O(|\mathcal{X}_{\mathcal{C}_k}|^2) t^f$ = $O(|\mathcal{X}_{\mathcal{C}_k}|^2) t^f$. Thus, in total, all cliques performing \crel in parallel takes $O(1) t_{\texttt{CRO}}^c + O(|\mathcal{X}_{\mathcal{C}_k}|^2) t^f$ time.

All in all, \texttt{DRM} takes $O(1) (t_{\texttt{DCP}}^c + t_{\texttt{DCP}}^s)$ + $O(1) t_{\texttt{CRO}}^c$ + $O(|\mathcal{X}_\mathcal{M}|^2) t^f$ time. In addition, each robot has four-round communications, including three rounds in \texttt{DCP} and one round in per clique \crel, and exchanges $3|\mathcal{N}_i|+ |\mathcal{C}_k|-1$ messages with $i\in \mathcal{C}_k$. Moreover, \texttt{DRM} performs $O(|\mathcal{N}_i|)$ operations for set intersections of each robot $i$ in \texttt{DCP} and $O(|\mathcal{X}_{\mathcal{C}_k}|^2)$ evaluations of objective function $f$ in \crel for each clique $\mathcal{C}_k$.    

\label{proof:runtime}
\end{proof}

\subsection{Proof of Theorem~\ref{thm:DRA}}

\begin{proof}
{We prove Theorem~\ref{thm:DRA}, i.e., \texttt{DRM}'s approximation bound, by following the steps of \cite[Proof of Theorem 1]{tzoumas2018resilient}.}

We introduce the notation:
$\mathcal{S}^{\star}$ denotes an optimal solution to Problem~\ref{pro:dis_resi_sub}. Given an action assignment $\mathcal{S}$ to all robots in $\mathcal{R}$, and a subset of robots $\mathcal{R}'$, we denote by $\mathcal{S}(\mathcal{R}')$ the actions of the robots in $\mathcal{R}'$ (i.e., the restriction of $\mathcal{S}'$ to $\mathcal{R}'$).  And vise versa: given an action assignment $\mathcal{S}'$ to a subset $\mathcal{R}'$ of robots, we let $\mathcal{R}(\mathcal{S})$ denote this subset (i.e., $\mathcal{R}(\mathcal{S}')=\mathcal{R}'$). Additionally, we let $\mathcal{S}_k\triangleq \mathcal{S}(\mathcal{C}_k)$; that is, $\mathcal{S}_k$ is the restriction of $\mathcal{S}$ to the clique $\mathcal{C}_k$ selected by \alg's line~1 ($k\in\{1,\ldots, K\}$); evidently, $\mathcal{S} = \bigcup_{k=1}^{K}\mathcal{S}_k$. Moreover, we let $\mathcal{S}_k^{b}$ correspond to bait actions chosen by \crel in $\mathcal{C}_k$, and $\mathcal{S}_k^{g}$ denote the greedy actions for the remaining robots in $\mathcal{C}_k$; that is, $\mathcal{S}_k = \mathcal{S}_k^{b} \cup \mathcal{S}_k^{g}$. If $\alpha \geq |\mathcal{C}_k|$, then $\mathcal{S}_k^{g} = \emptyset$. Henceforth, we let $\mathcal{S}$ be the action assignment given by \alg to all robots in $\mathcal{R}$.  Also, we let $\mathcal{W}$ be remaining robots after the attack $\mathcal{A}^\star(\mathcal{S})$; i.e., $\mathcal{W} \triangleq\mathcal{R}\setminus \mathcal{R}(\mathcal{A}^\star(\mathcal{S}))$. Further, we let $\mathcal{W}_k\triangleq \mathcal{W}\cap \mathcal{C}_k$ be remaining robots in $\mathcal{C}_k$, $\mathcal{W}_k^b\triangleq \mathcal{W}_k\cap \mathcal{R}(\mathcal{S}_k^b)$ be remaining robots with bait actions in $\mathcal{C}_k$, and  $\mathcal{W}_k^g\triangleq \mathcal{W}_k\cap \mathcal{R}(\mathcal{S}_k^g)$ be remaining robots with greedy actions in $\mathcal{C}_k$. Then we have $\mathcal{S}_k^{b}\setminus \mathcal{W}_k^b$ and $\mathcal{S}_k^{g}\setminus \mathcal{W}_k^g$ to denote the attacked robots with bait actions and with greedy actions in $\mathcal{C}_k$, respectively. Then, we have the number of attacks in $\mathcal{C}_k$ as $|\mathcal{S}_k^{b}\setminus \mathcal{W}_k^b| + |\mathcal{S}_k^{g}\setminus \mathcal{W}_k^g|$ and 
\begin{equation}
    |\mathcal{S}_k^{b}\setminus \mathcal{W}_k^b| + |\mathcal{S}_k^{g}\setminus \mathcal{W}_k^g| \leq \alpha,
    \label{eqn:gre_alpha}
\end{equation}
since the total number of attacks for the robot team is $\alpha$. Also, we have 
\begin{equation}
   |\mathcal{S}_k^{b}\setminus \mathcal{W}_k^b| + |\mathcal{W}_k^b|  = \alpha,
    \label{eqn:eq_alpha}
\end{equation}
since the number of robots with bait actions in $\mathcal{C}_k$ is $\alpha$. With Eqs.~\ref{eqn:gre_alpha} and \ref{eqn:eq_alpha}, we have
\begin{equation}
   |\mathcal{W}_k^b| \geq |\mathcal{S}_k^{g}\setminus \mathcal{W}_k^g|.
    \label{eqn:gre_remo}
\end{equation}
With Eq.~\ref{eqn:gre_remo}, we let $\mathcal{W}_k^{b'}$ denote the remaining robots in $\mathcal{W}_k^{b}$ after removing from it any subset of robots with cardinality $|\mathcal{R}(\mathcal{S}^g_k)\setminus \mathcal{W}^g_k|$. Then $|\mathcal{W}_k^{b} \setminus \mathcal{W}_k^{b'}| = |\mathcal{R}(\mathcal{S}^g_k)\setminus \mathcal{W}^g_k|$ and robots in $\mathcal{W}_k^{b} \setminus \mathcal{W}_k^{b'}$ have bait actions and robots in $\mathcal{R}(\mathcal{S}^g_k)\setminus \mathcal{W}^g_k$ have greedy actions. Then, $\mathcal{W}_k^{b} \setminus \mathcal{W}_k^{b'}$ can be used to compensate for the attacked robots $\mathcal{R}(\mathcal{S}^g_k)\setminus \mathcal{W}^g_k$ in $\mathcal{C}_k$. 

Now the proof follows from the steps:
\begin{align}
&f(\mathcal{S}^{}\setminus \mathcal{A}^{\star}(\mathcal{S})) \geq (1-\nu_{f}) \sum_{r\in \mathcal{W}}f(\mathcal{S}(r))
\label{ineq:thm1_curv_ratio1}\\
& = (1-\nu_{f}) \sum_{k=1}^{K} \sum_{r\in \mathcal{W}_k}f(\mathcal{S}(r)) \label{ineq:thm1_curv_ratio21}\\
& = (1-\nu_{f}) \sum_{k=1}^{K}\left[ \sum_{r\in \mathcal{W}_k^b}f(\mathcal{S}(r)) +  \sum_{r\in \mathcal{W}_k^g}f(\mathcal{S}(r))\right] \label{ineq:thm1_curv_ratio22}\\
& = (1-\nu_{f}) \sum_{k=1}^{K}\left[ \sum_{r\in \mathcal{W}_k^{b'}}f(\mathcal{S}(r))+ \right.\nonumber\\
&\hspace*{1cm} \left.\sum_{r\in {\mathcal{W}_k^{b} \setminus \mathcal{W}_k^{b'}}} f(\mathcal{S}(r)) + \sum_{r\in \mathcal{W}_k^g}f(\mathcal{S}(r))\right]
\label{ineq:thm1_curv_ratio231}\\
& \geq (1-\nu_{f}) \sum_{k=1}^{K}\left[ \sum_{r\in \mathcal{W}_k^{b'}}f(\mathcal{S}(r))+ \right.\nonumber\\
&\hspace*{1cm} \left.\sum_{r\in {\mathcal{R}(\mathcal{S}^g_k)\setminus\mathcal{W}_k^g}} f(\mathcal{S}(r)) + \sum_{r\in \mathcal{W}_k^g}f(\mathcal{S}(r))\right]
\label{ineq:thm1_curv_ratio23}\\
& =(1-\nu_{f}) \sum_{k=1}^{K}\left[\sum_{r\in \mathcal{W}_k^{b'}}f(\mathcal{S}(r)) +\sum_{r\in {\mathcal{R}(\mathcal{S}^g_k)}} f(\mathcal{S}(r))\right]
\label{ineq:thm1_curv_ratio24}\\
& \geq (1-\nu_{f}) \sum_{k=1}^{K}\left[\sum_{r\in \mathcal{W}_k^{b'}}f(\mathcal{S}(r)) + f(\mathcal{S}^g_k)\right]
\label{ineq:thm1_curv_ratio259}\\
& \geq (1-\nu_{f}) \sum_{k=1}^{K}\left[\sum_{r\in \mathcal{W}_k^{b'}}f(\mathcal{S}^\star(r)) + \frac{1}{2}f(\mathcal{S}^\star(\mathcal{R}(\mathcal{S}_k^g)))\right]
\label{ineq:thm1_curv_ratio26}\\
& \geq \frac{1-\nu_{f}}{2} \sum_{k=1}^{K} f(\mathcal{S}^{\star}(\mathcal{W}_k)) \label{ineq:thm1_curv_ratio25} \\
& \geq \frac{1-\nu_{f}}{2} f(\mathcal{S}^{\star}(\mathcal{W})) \label{ineq:thm1_curv_ratio277776}\\
&\geq \frac{1-\nu_{f}}{2} f(\mathcal{S}^{\star}\setminus \mathcal{A}^{\star}(\mathcal{S}^{\star})).~\label{ineq:thm1_curv_ratio5}
\end{align}

Ineq.~\eqref{ineq:thm1_curv_ratio1} follows from the definition of $\nu_f$ (see~\cite[Proof of Theorem 1]{tzoumas2018resilient}). Eqs.~\eqref{ineq:thm1_curv_ratio21} and \eqref{ineq:thm1_curv_ratio22} follow from the notation we introduced above. Eq.~\eqref{ineq:thm1_curv_ratio231} holds since $\mathcal{W}_k^{b} = \mathcal{W}_k^{b'} \cup (\mathcal{W}_k^{b} \setminus \mathcal{W}_k^{b'})$. Ineq.~\eqref{ineq:thm1_curv_ratio23} holds since a) $|\mathcal{W}_k^{b} \setminus \mathcal{W}_k^{b'}| = |\mathcal{R}(\mathcal{S}^g_k)\setminus \mathcal{W}^g_k|$; b) robots in $\mathcal{W}_k^{b} \setminus \mathcal{W}_k^{b'}$ have bait actions and robots in $\mathcal{R}(\mathcal{S}^g_k)\setminus \mathcal{W}^g_k$ have greedy actions, such that for any $r \in \mathcal{W}_k^{b} \setminus \mathcal{W}_k^{b'}$ and $r'\in \mathcal{R}(\mathcal{S}^g_k)\setminus \mathcal{W}^g_k$, we have $f(\mathcal{S}(r)) \geq f(\mathcal{S}(r'))$.
Eq.~\eqref{ineq:thm1_curv_ratio24} holds from the notation. 
Ineq.~\eqref{ineq:thm1_curv_ratio259} holds by the submodularity of $f$, which implies $f(\mathcal{A})+f(\mathcal{B})\geq f(\mathcal{A} \cup \mathcal{B})$ for any sets $\mathcal{A},\mathcal{B}$~\cite{nemhauser1978analysis}. 
Ineq.~\eqref{ineq:thm1_curv_ratio26} holds since a) with respect to the left term in the sum, the robots in the sum correspond to robots whose actions are baits; and b) with respect to the right term in the sum, the greedy algorithm that has assigned the actions $\mathcal{S}_k^g$ guarantees at least $1/2$ the optimal~\cite{fisher1978analysis}.  Ineq.~(11) holds again due to the submodularity of $f$, as above.  The same for ineq.~(12). Ineq.~(13) follows from \cite[Proof of Theorem 1]{tzoumas2018resilient} because of the worst-case removal.\qedhere
\label{proof:appro}
\end{proof}

\subsection{{Myopic optimization yields tighter approximation performance, yet worse practical performance}} \label{sec:clarify_bound}

{The myopic Algorithm~\ref{alg:MM}, according to which each robot selects its best action independently, guarantees a tighter approximation bound than that of $\texttt{DRM}$:

\begin{algorithm}[t]
\caption{Myopic algorithm for Problem~\ref{pro:dis_resi_sub}.}
\begin{algorithmic}[1]
\REQUIRE
    Robots' available actions $\mathcal{X}_i$,  $i\in \mathcal{R}$;
    monotone and submodular $f$.
\ENSURE  Robots' actions $\mathcal{S}$.
\STATE $\mathcal{S} \leftarrow \emptyset$;
\FOR{$i\in\mathcal{R}$}
\label{MM:for_robot_start}
    \STATE $s \leftarrow \text{argmax}_{s\in\mathcal{X}_i} f(s)$; \label{MM:argmax}
    
    \STATE $\mathcal{S} \leftarrow \mathcal{S} \cup \{s\}$;
    \label{MM:union}
\ENDFOR \label{MM: for_robot_end}
\RETURN $\mathcal{S}$.
\end{algorithmic}
\label{alg:MM}
\end{algorithm}

\begin{theorem}[Approximation performance of Algorithm~\ref{alg:MM}]~\label{thm:mm}
Algorithm~\ref{alg:MM} returns a feasible action-set $\mathcal{S}$ such that
\begin{align}\label{eq:appro_mm}
\begin{split}
\frac{f(\mathcal{S}\setminus \mathcal{A}^\star(\mathcal{S}))}{f^\star} \geq  1-\nu_f.
\end{split}
\end{align}
\end{theorem}

\begin{proof}
We split $\mathcal{S}$ generated by Algorithm~\ref{alg:MM} into $\mathcal{S}_1$ and $\mathcal{S}_2$, with $\mathcal{S}_1$ denoting the action set selected by the top $\alpha$ robots, and $\mathcal{S}_2$ denoting the action set selected by the remaining $|\mathcal{R}|-\alpha$ robots. Denote $\mathcal{S}_2^\star$ as the action set selected by $\mathcal{R}(\mathcal{S}_2)$ to maximize $f(\mathcal{A}), ~\mathcal{A} \in \mathcal{I}, \mathcal{A}\subseteq \mathcal{X}(\mathcal{R}(\mathcal{S}_2))$. 

\begin{align}
&f(\mathcal{S}^{}\setminus \mathcal{A}^{\star}(\mathcal{S})) \geq (1-\nu_{f}) \sum_{s\in \mathcal{S}_2}f(s) \label{ineq:alg5_curv1}\\
& \geq (1-\nu_{f}) \sum_{s\in \mathcal{S}_2^\star}f(s^\star) \label{ineq:alg5_curv2}\\
& \geq (1-\nu_{f}) f(\mathcal{S}_2^\star) \label{ineq:alg5_curv3}\\ 
& \geq (1-\nu_{f}) f(\mathcal{S}^{\star}\setminus \mathcal{A}^{\star}(\mathcal{S}^{\star})) \label{ineq:alg5_curv4}. 
\end{align}

Ineq.~\eqref{ineq:alg5_curv1} follows from the definition of $\nu_f$ (see~\cite[Proof of Theorem 1]{tzoumas2018resilient}). Ineq.~\eqref{ineq:alg5_curv2} holds since Algorithm~\ref{alg:MM} selects the best action for each robot (Algorithm~\ref{alg:MM}, line~\ref{MM:argmax}). Ineq.~\eqref{ineq:alg5_curv3} holds due to the submodularity of $f$. Ineq.~\ref{ineq:alg5_curv4} follows from \cite[Proof of Theorem 1]{tzoumas2018resilient} because of the worst-case removal.\qedhere
\label{prof:mm}
\end{proof}

However, Algorithm~\ref{alg:MM} performs in practice worse than $\texttt{DRM}$ because: (i) it chooses actions for each robot independently of the actions of the rest of the robots (instead, \alg takes into account other robots' actions to intentionally reduce the performance redundancy among these robots); and (ii) it is equivalent to \crel \textit{but} under the assumption the number of attacks is equal to the number of robots, which is, evidently, conservative; \alg instead is less conservative assuming at most $\alpha$ attacks per clique).}

An evaluation of the practical performance of Algorithm~\ref{alg:MM} in comparison to \alg's is made in Fig.~\ref{fig:compare_mm}. The figures clearly show that in both small-scale and large-scale cases, $\texttt{DRM}$ outperforms Algorithm~\ref{alg:MM}. 

\begin{figure}[th!]
\centering{
\subfigure[$N=20$ ]{\includegraphics[width=0.49\columnwidth]{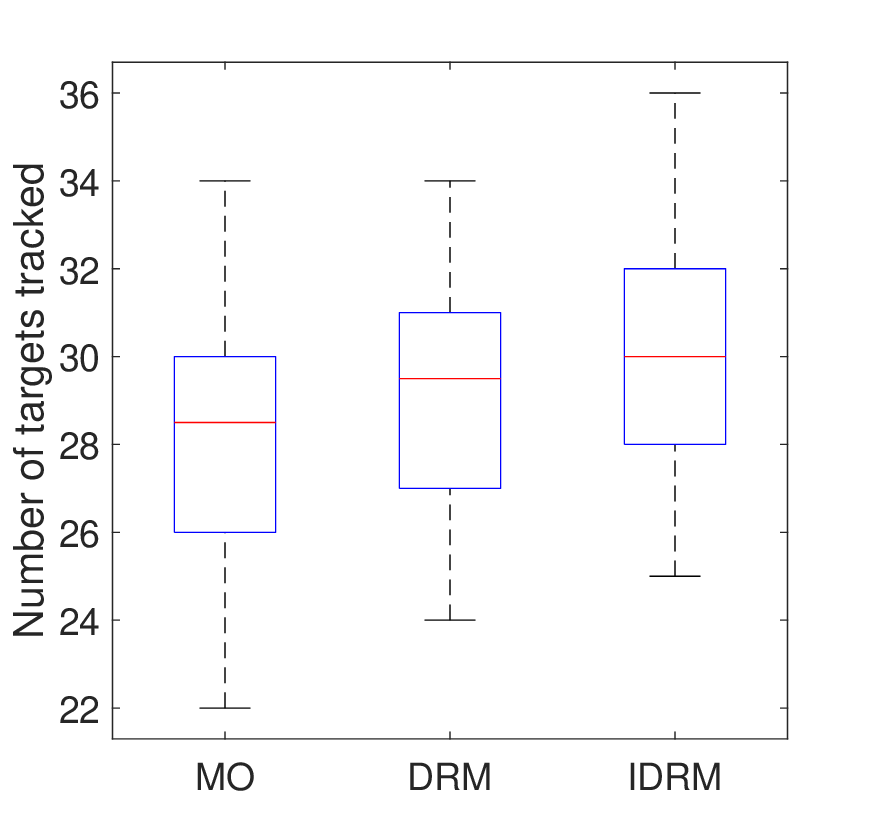}}~~
\subfigure[$N=100$]{\includegraphics[width=0.49\columnwidth]{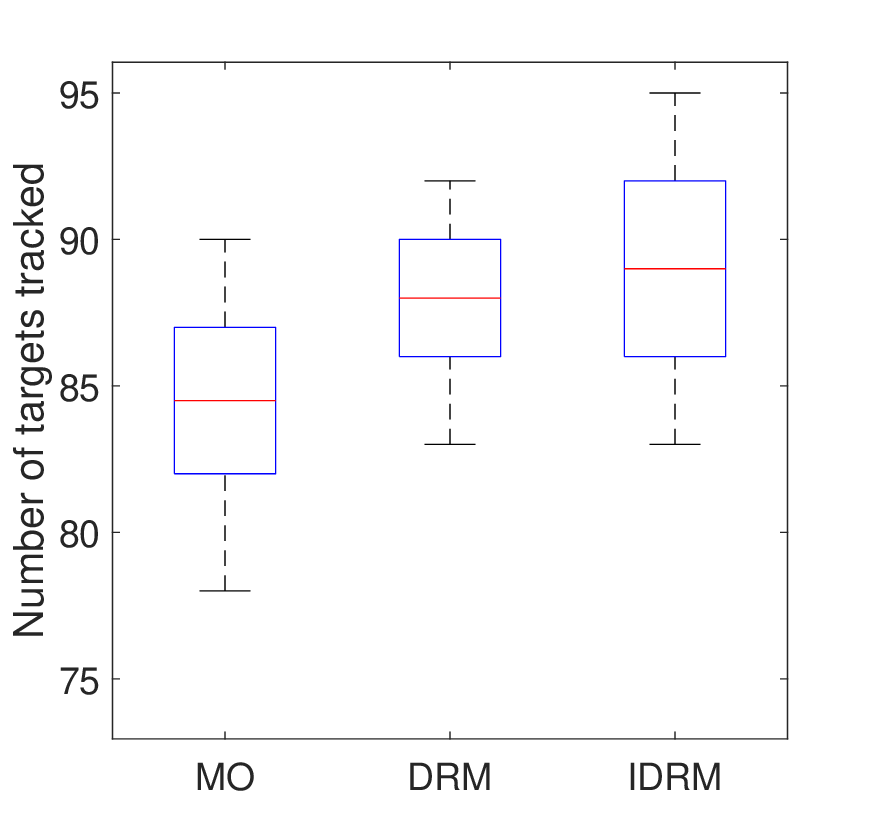}}
\caption{{MATLAB evaluations: comparison of number of targets covered by Algorthm~\ref{alg:MM} (called \texttt{MO}), \texttt{DRM}, \texttt{IDRM} in (a) small-scale and (b) large-scale cases. The simulation settings follow the Matlab simulation setup of Section~\ref{sec:simulation}-A.  In (a) small-scale case with $N=20$, $\alpha=4$, and $r_c=120$, the number of target tracked by these three algorithms are calculated after applying 4 worst-case attacks. In (b) large-scale case with $N=100$, $\alpha=10$, and $r_c=70$, the number of target tracked by these three algorithms are calculated after applying 10 greedy attacks. In (a) and (b), the representations of each box's components follow the corresponding explanations in the caption of Fig.~\ref{fig:compare_gazebo}. } 
\label{fig:compare_mm}}}
\end{figure}

\bibliographystyle{IEEEtran}
\bibliography{IEEEabrv,references}

\begin{IEEEbiography}[{\includegraphics[width=1in,height=1.25in,clip,keepaspectratio]{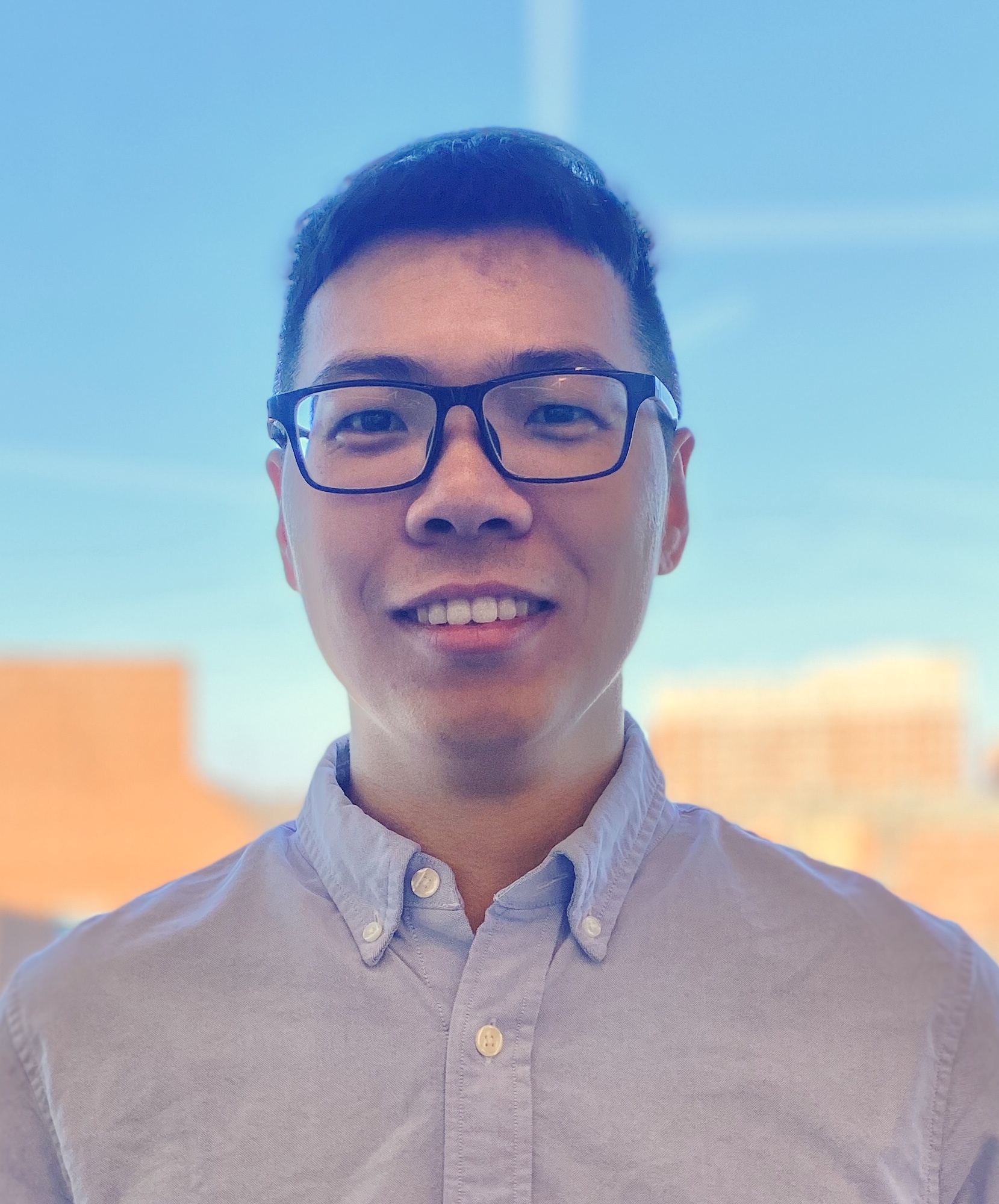}}]{Lifeng Zhou} is currently a Postdoctoral Researcher in the GRASP Lab at the University of Pennsylvania. He received his Ph.D. degree in Electrical \& Computer Engineering at Virginia Tech in 2020. He obtained his master’s degree in Automation from Shanghai Jiao Tong University, China in 2016, and his Bachelor's degree in Automation from Huazhong University of Science and Technology, China in 2013. 

His research interests include multi-robot coordination, approximation algorithms, combinatorial optimization, model predictive control, graph neural networks, and resilient, risk-aware decision making.
\end{IEEEbiography}

\begin{IEEEbiography}[{\includegraphics[width=1in,height=1.25in,clip,keepaspectratio]{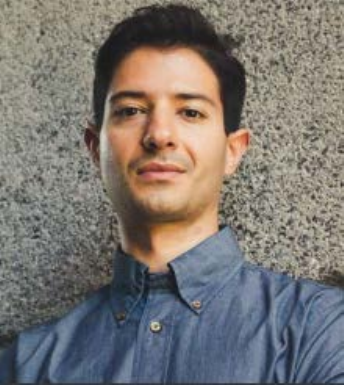}}]{Vasileios Tzoumas} received his Ph.D. in Electrical and Systems Engineering at the University of Pennsylvania (2018). He holds a Master of Arts in Statistics from the Wharton School of Business at the University of Pennsylvania (2016); a Master of Science in Electrical Engineering from the University of Pennsylvania (2016); and a diploma in Electrical and Computer Engineering from the National Technical University of Athens (2012). Vasileios is an Assistant Professor in the Department of Aerospace Engineering, University of Michigan, Ann Arbor. Previously, he was a research scientist in the Department of Aeronautics and Astronautics, and the Laboratory for Information and Decision Systems (LIDS), at the Massachusetts Institute of Technology (MIT). In 2017, he was a visiting Ph.D. student at the Institute for Data, Systems, and Society (IDSS) at MIT. Vasileios works on control, learning, and perception, as well as combinatorial and distributed optimization, with applications to robotics, cyber-physical systems, and self-reconfigurable aerospace systems.  He aims for trustworthy collaborative autonomy.  His work includes foundational results on robust and adaptive combinatorial optimization, with applications to multi-robot information gathering for resiliency against robot failures and adversarial removals. Vasileios is a recipient of the Best Paper Award in Robot Vision at the 2020 IEEE International Conference on Robotics and Automation (ICRA), of a Honorable Mention at  the 2020 IEEE Robotics and Automation Letter (RA-L), and was a Best Student Paper Award finalist at the 2017 IEEE Conference in Decision and Control (CDC).
\end{IEEEbiography}

\begin{IEEEbiography}[{\includegraphics[width=1in,height=1.25in,clip,keepaspectratio]{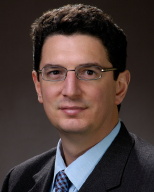}}]
{George J.~Pappas} (S'90-M'91-SM'04-F'09) received the Ph.D.~degree in electrical engineering and computer sciences from the University of California, Berkeley, CA, USA, in  1998. He is currently the Joseph Moore Professor and Chair of the Department of Electrical and Systems Engineering, University of Pennsylvania, Philadelphia, PA, USA. He  also holds a secondary appointment with the Department of Computer and Information Sciences and the Department of Mechanical Engineering and Applied Mechanics. He is a Member of the GRASP Lab and the PRECISE Center. He had previously served as the Deputy Dean for Research with the School of Engineering and Applied Science. His research interests include control theory and, in particular, hybrid systems, embedded systems, cyberphysical systems, and hierarchical and distributed control systems, with applications to unmanned aerial vehicles, distributed robotics, green buildings, and biomolecular networks. Dr. Pappas has received various awards, such as the Antonio Ruberti Young Researcher Prize, the George S. Axelby Award, the Hugo Schuck Best Paper Award, the George H. Heilmeier Award, the National Science Foundation PECASE award and numerous best student papers awards.
\end{IEEEbiography}

\begin{IEEEbiography}[{\includegraphics[width=1in,height=1.25in,clip,keepaspectratio]{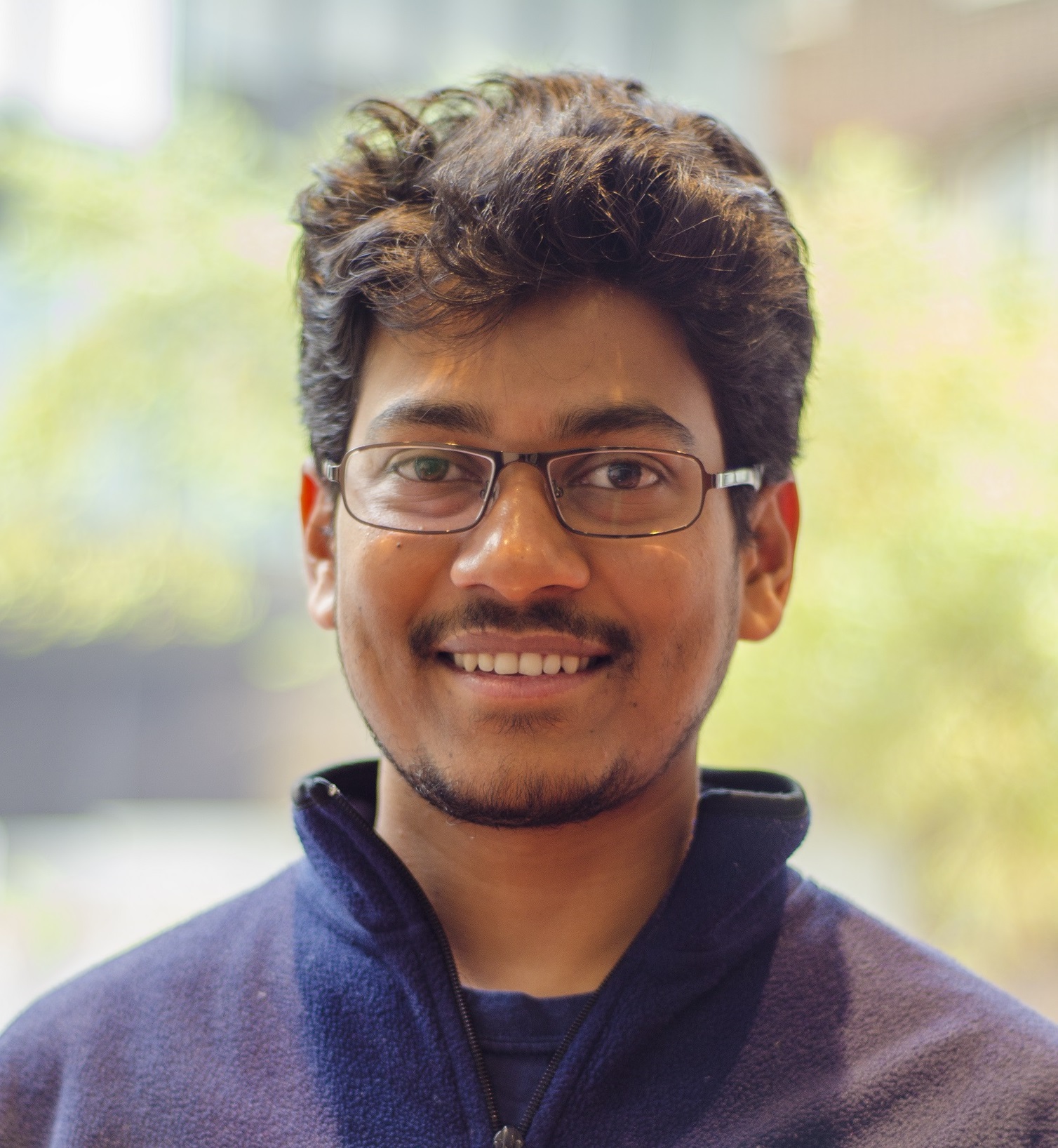}}]{Pratap Tokekar} is an Assistant Professor in the Department of Computer Science at the University of Maryland. Previously, he was a Postdoctoral
Researcher at the GRASP lab of University of
Pennsylvania. Between 2015 and 2019, he was an Assistant Professor at the Department of Electrical and Computer Engineering at Virginia Tech. He obtained his Ph.D. in Computer Science from the University of Minnesota in 2014 and Bachelor of Technology degree in Electronics and Telecommunication from College of Engineering Pune, India in
2008. He is a recipient of the NSF CISE Research Initiation Initiative award and an Associate Editor for the IEEE Robotics and Automation Letters and Transactions of Automation Science and Engineering.
\end{IEEEbiography}

\end{document}